\documentclass[journal,transmag]{IEEEtran}
\ifCLASSINFOpdf

\fi

\let\oldtitle\title
\renewcommand\title[1]{%
    \begingroup
        \providecommand{\ttlit}{}%
        \renewcommand{\ttlit}[1]{}%
        \providecommand{\titlenote}{}%
        \renewcommand{\titlenote}[1]{}%
        \hypersetup{pdftitle={#1}}%
        \def\thetitle{#1}%
        \pdfbookmark[0]{#1}{title}
    \endgroup
    \oldtitle{#1}%
}

\usepackage{calc}

\usepackage[shortcuts]{extdash}

\usepackage{graphicx}

\usepackage{booktabs}

\usepackage{verbatim}

\usepackage{amsmath}
\usepackage{psfrag}
\usepackage{graphics} % for pdf, bitmapped graphics files
\usepackage{epsfig} % for postscript graphics files
\usepackage{mathptmx} % assumes new font selection scheme installed
\usepackage{times} % assumes new font selection scheme installed
\usepackage{amsmath} % assumes amsmath package installed
\usepackage{amssymb}  % assumes amsmath package installed
\usepackage{ifpdf}
\usepackage{cite}
\usepackage{fp}
\usepackage{algpseudocode}
\usepackage{algorithm}
\makeatletter
\def\citet{\@ifstar{\citetstar}{\citetnostar}}
\def\Citet{\@ifstar{\Citetstar}{\Citetnostar}}
\def\citetnostar{\@ifnextchar[{\squarecitet}{\simplecitet}}
\def\squarecitet[#1]{\@ifnextchar[{\twocitet[#1]}{\onecitet[#1]}}
\def\Citetnostar{\@ifnextchar[{\squareCitet}{\simpleCitet}}
\def\squareCitet[#1]{\@ifnextchar[{\twoCitet[#1]}{\oneCitet[#1]}}
\def\citetstar{\@ifnextchar[{\squarecitetstar}{\simplecitetstar}}
\def\squarecitetstar[#1]{\@ifnextchar[{\twocitetstar[#1]}{\onecitetstar[#1]}}
\def\Citetstar{\@ifnextchar[{\squareCitetstar}{\simpleCitetstar}}
\def\squareCitetstar[#1]{\@ifnextchar[{\twoCitetstar[#1]}{\oneCitetstar[#1]}}
\makeatother
\def\simplecitet#1{\citeauthor{#1}~\citep{#1}}
\def\onecitet[#1]#2{\citeauthor{#2}~\citep[#1]{#2}}
\def\twocitet[#1][#2]#3{\citeauthor{#3}~\citep[#1][#2]{#3}}
\def\simplecitetstar#1{\citeauthor*{#1}~\citep{#1}}
\def\onecitetstar[#1]#2{\citeauthor*{#2}~\citep[#1]{#2}}
\def\twocitetstar[#1][#2]#3{\citeauthor*{#3}~\citep[#1][#2]{#3}}
\def\simpleCitet#1{\Citeauthor{#1}~\citep{#1}}
\def\oneCitet[#1]#2{\Citeauthor{#2}~\citep[#1]{#2}}
\def\twoCitet[#1][#2]#3{\Citeauthor{#3}~\citep[#1][#2]{#3}}
\def\simpleCitetstar#1{\Citeauthor*{#1}~\citep{#1}}
\def\oneCitetstar[#1]#2{\Citeauthor*{#2}~\citep[#1]{#2}}
\def\twoCitetstar[#1][#2]#3{\Citeauthor*{#3}~\citep[#1][#2]{#3}}
\usepackage{etex} % Should be first line
\usepackage[cmex10]{mathtools}             % Loads and extends amsmath
\usepackage{amsfonts,amssymb}
\usepackage{mathrsfs}                      % Gives us \mathscr for \setset
\usepackage{amsthm}
\usepackage{parskip}
\usepackage{amssymb}
\usepackage{fancyref} 
\usepackage{varioref}
\usepackage[utf8]{inputenc} 
\labelformat{subfigure}{\thefigure\textup{(#1)}}
\labelformat{equation}{\textup{(#1)}}
\usepackage[%
        pagebackref=false,
        bookmarks=true,bookmarksopen=true,
        bookmarksnumbered=true,
        breaklinks=true,
        colorlinks=true,
        anchorcolor=blue,citecolor=blue,
        urlcolor=blue,linkcolor=blue,filecolor=blue,
        menucolor=blue,
    ]{hyperref}
\usepackage{breakurl}
\usepackage{doi}

\usepackage{subfigure}
\usepackage{booktabs}
\usepackage{pstricks}
\usepackage{float}
\usepackage{subfigure}
\usepackage{amsthm}
\usepackage{url}
\usepackage{graphicx}
\usepackage[mathscr]{eucal}
\usepackage{mathbbol}
\usepackage[inline]{enumitem}
\usepackage{bbold}
\usepackage{tikz}
\usetikzlibrary{positioning}
\usetikzlibrary{shapes.geometric}
\usetikzlibrary{shapes.misc}
\usetikzlibrary{arrows}
\usetikzlibrary{intersections}
\usepackage{amsmath}
\usepackage{amssymb}
\usepackage{xstring}
\usepackage{pgfplots} 
\makeatletter
\@ifpackageloaded{hyperref}{%
\@ifpackageloaded{amsmath}{%
\newcommand{\AMShreffix}[1]{%
        \expandafter\let\csname AMShreffix#1\expandafter\endcsname%
                \csname #1\endcsname%
        \expandafter\renewcommand\csname #1\endcsname{%
                \@hyper@itemfalse\csname AMShreffix#1\endcsname}}
\AtBeginDocument{%
\AMShreffix{equation}
\AMShreffix{align}
\AMShreffix{alignat}
\AMShreffix{flalign}
\AMShreffix{gather}
\AMShreffix{multline}
}}{}}{}
\makeatother

\def\figurename{Fig.}

\newcommand{\overbar}[1]{\mkern 1.5mu\overline{\mkern-1.5mu#1\mkern-1.5mu}\mkern 1.5mu}
\newcommand{\pmon}{p_{\mathrm{mon}}}
\newcommand{\mon}{\mathrm{mon}}

\newcommand{\ssimpnf}{\ssimpn_{\mathrm{full}}}

\newcommand{\gensimp}{\mathcal{S}}
\newcommand{\genpoly}{\mathcal{P}}

\newcommand{\hyp}[1] {\mathcal{H}_{#1}}

\newcommand{\vm}{\mathbb{Ver}}
\newcommand{\Vol}{\mathrm{Vol}} 
\newcommand{\favn} {\mathcal{F}}

\newcommand{\unfav} {\mathcal{\overbar{F}}} 

\newcommand{\unfavn} {\unfav}

\newcommand{\hypn}{\mathcal{H}}

\newcommand{\slen}{\mathrm{slen}}

\newcommand{\upos}[1]{x_{#1}}
\newcommand{\pos}[1]{  \mathtt{x}_{#1}}
\newcommand{\uposvec}{\mathbf{x}}
\newcommand{\posvec}{\mathbf{\mathtt{x}}}

\newcommand{\psimpn} {\mathcal{P}} 

\newcommand{\ssimpcf}{\mathcal{S}_{\rm{CF}}}

\newcommand{\rvupos}[1] {X_{#1}}
\newcommand{\rvuposvec} {\mathbf{X}}
\newcommand{\rvpos}[1] {\mathtt{ X}_{#1}}

\newcommand{\rvsvec} {\mathbf{S}}

\newcommand{\jpos} { f_{\mathtt{X}} (\posvec)}

\newcommand{\mpos}[1] {f_{\mathtt{X}_{#1}} (\pos{#1})}   % Spring: changed _X^{rightarrow} to _\mathtt{X}
\newcommand{\msla}[1] {f_{S_{#1}} (\si{#1})}

\newcommand{\avec}{\mathbf{a}}

\newcommand{\G} {\mathcal{G}}
\newcommand{\fpar}{g}
\newcommand{\Fpar}{G}

\newcommand{\Poi}{\mathrm{Poi}}

\newcommand{\sla} {s} 
\newcommand{\si}[1]{s_{#1}}
\newcommand{\svec} {\mathbf{s}}

\newcommand{\orvsi}[1]{ \mathtt{S}_{#1}}

\newcommand{\orsvec}{\mathbf{\mathtt{S}}}

\newcommand{\Bo}{\mathcal{B}}

\newcommand{\fsi}[1] {\tilde{s}_{#1}}
\newcommand{\fsla}{\tilde{s}}
\newcommand{\fsvec}{\tilde{\svec}}

\newcommand{\ssimpn}{\mathcal{S}}
\newcommand{\ssimp}{\mathcal{S}}

\newcommand{\rvsla}[1] {S_{#1}}
\newcommand{\jsla} {f_{\mathbf{S}} (\mathbf{s})}

\newcommand{\rvfsla}[1] {\tilde{\rvsla{i}}}
\newcommand{\rvfpos}[1] {\tilde{\rvpos{i}}}

\newcommand{\nmin}{n_{\min}}

\newcommand{\eye}{\mathbb{I}}

\newcommand{\0}{\mathbf{0}}
\newcommand{\1}{\mathbf{1}}

\newcommand{\rad}{r}
\newcommand{\dia}{D}

\newcommand{\nor}{n}

\newcommand{\sen}{\mathrm{sen}}
\newcommand{\psen}{p_{\mathrm{sen}}}
\newcommand{\sat}{\con}
\newcommand{\con}{d}

\newcommand{\pcon} {p_{\mathsf{con}}}

\newcommand{\Rob}{\mathcal{R}}

\newcommand{\R}{\mathbb{R}}

\newcommand{\vvec}{\mathbf{v}}

\newcommand{\ind} {\mathbb{1}}

\newcommand{\Exp} {\mathbb{E}}
 
\newcommand{\Var} {\mathbb{V}}

\newcommand{\todo}[1] {{\color{magenta} #1} }
\newcommand{\Prob}{\mathbb{P}}

\newcommand{\cmp}{\mathrm{cmp}}

\newcommand{\dgr}{\mathrm{deg}}

\newcommand{\Nmon}{\mathrm{N}_\mon}
\newcommand{\Ncon}{\mathrm{N}_{\rm con}}
\newcommand{\Nsen}{\mathrm{N}_{\rm sen}}

\newtheorem{theorem}{Theorem}

\newtheorem{lemma}{Lemma}

\def\hlinewd#1{%
\noalign{\ifnum0=`}\fi\hrule \@height #1 %
\futurelet\reserved@a\@xhline}

\begin{document}
\title{Design of Stochastic Robotic Swarms for Target Performance Metrics in Boundary Coverage Tasks}
%Boundary Coverage Performance Metrics}
%Analytical Results and Algorithms for Stochastic Boundary Coverage by Robotic Swarms}  %: Algorithms, Analysis, and Applications}


\author{\IEEEauthorblockN{Ganesh P. Kumar and
        Spring Berman
        }
     {
	 School for Engineering of Matter, Transport and Energy, Arizona State University, 
	Tempe, AZ 85287 USA }
}

\markboth{}%  % Journal of \LaTeX\ Class Files,~Vol.~11, No.~4, December~2012
{Shell \MakeLowercase{\textit{et al.}}: Bare Demo of IEEEtran.cls for Journals}

\IEEEtitleabstractindextext{%
\begin{abstract}

In this work, we analyze \textit{stochastic coverage schemes} (SCS) for robotic swarms in which the robots randomly attach to a one-dimensional boundary of interest using local communication and sensing, without relying on global position information or a map of the environment.   Robotic swarms may be required to perform boundary coverage in a variety of applications, including environmental monitoring, collective transport, disaster response, and nanomedicine. We present a novel analytical approach to computing and designing the statistical properties of the communication and sensing networks that are formed by random robot configurations on a boundary.  We are particularly interested in the event that a robot configuration forms a connected communication network or maintains continuous sensor coverage of the boundary.  Using tools from  order statistics, random geometric graphs, and computational geometry, we derive formulas for properties of the random graphs generated by robots that are independently and identically distributed along a boundary.  We also develop order-of-magnitude estimates of these properties based on Poisson approximations and threshold functions.  For cases where the SCS generates a uniform distribution of robots along the boundary, we apply our analytical results to develop a procedure for computing the robot population size, diameter, sensing range, or communication range that  yields a random communication network or sensor network with desired properties.
    % edit
\end{abstract}

}
% http://www.ieee-ras.org/images/publications/t-ase/T-RO_Keywords.pdf
%Select a minimum of 2 up to a maximum of 5 (recommended) keywords during one of the submission steps at the T-RO Papercept site ras.papercept.net/journals/tro. Keywords can be chosen from those predefined within the RAS subject areas (the ones used for ICRA conferences) or typed in freely (for at most 2 keywords). 

\maketitle

\IEEEdisplaynontitleabstractindextext

\section{Introduction}
\label{Sec:Intro}
Robotic swarms have the potential to collectively perform tasks over very large domains and time scales, succeeding even in the presence of failures, errors, and disturbances.  Low-cost miniature autonomous robots for swarm applications are currently being developed as a result of recent advances in computing, sensing, actuation, power, control, and manufacturing technologies.  In addition, micro-nanoscale platforms such as DNA machines, synthetic bacteria, magnetic materials, and nanoparticles are being designed for biomedical and manufacturing applications \cite{Mavroidis2013,Diller2013,bauer201425th}, in which they would need to be deployed in massive numbers.  Swarm robotic platforms have limited onboard power that support only local sensing capabilities and local inter-robot communication and may preclude the use of GPS, or they may operate in GPS-denied environments.

In this work, we consider {\it boundary coverage} (BC) tasks for swarms of such resource-constrained robots.  We define a boundary coverage scheme (BCS) as a process by which multiple robots autonomously allocate themselves around the boundary of an object or region of interest.   Robotic swarms may be required to perform boundary coverage for mapping, exploration, environmental monitoring, surveillance tasks such as perimeter patrolling, and disaster response tasks such as cordoning off a hazardous area or extinguishing a fire.  Another motivating application is collective payload transport \cite{KubeBonabeau00,OPGCMBD09,Chen2015} for automated manipulation and assembly in uncertain, unstructured environments.  Furthermore, boundary coverage behaviors will need to be controlled in micro- and nanoscale systems that are designed for micro object manipulation, molecular imaging, drug and gene delivery, therapeutics, and diagnostics \cite{Tong2013, Hauert2013}. For example, nanoparticles that are designed for drug delivery and imaging can be coated with ligands and antibodies in a specific way to facilitate selective binding to tumor cell surfaces \cite{Sinha2006, Wang2012}.

We focus on \textit{stochastic coverage schemes} (SCS), in which the robots occupy {\it random} positions along a boundary.  Since swarm robotic platforms cannot perform precise navigation and localization, randomness in their motion will arise from noise due to sensor and actuator errors.  Even if the robots attempted to position themselves at equidistant locations along the boundary, noise in their odometry would introduce uncertainty into their resulting positions, such that each one would be distributed according to a Gaussian \cite{long2013banana}.  In addition, the robots may only encounter the boundary through local sensing during exploration of an unknown environment, which will introduce uncertainty in the locations of their interactions with the boundary.  In swarm applications at the nanoscale, the effects of Brownian motion and chemical interactions will contribute further sources of stochasticity in boundary coverage.

We address the problem of designing parameters of a robotic swarm that will produce {\it desired statistical properties} of the communication and sensing networks that are formed by random robot configurations around a boundary.  These parameters include the swarm population size and each robot's physical dimensions and sensing and communication radii, which we assume are identical for all robots.  The desired properties pertain to the distribution of robots around a boundary that result from an SCS; here we do not consider the process by which the robots arrive at the boundary.

The novelty of our approach lies in our integration of a variety of analytical tools to characterize properties of SCS, as well as our application of these results to design robotic swarms for desired SCS properties.  We devise a geometric approach to compute properties of the random graph generated by robots that have attached to the boundary independently of one another, in the case where robots may overlap.  We adapt this approach to the case where robots avoid conflicts, i.e. they do not overlap with each other on the boundary.   We derive both closed-form expressions  of SCS properties and estimates of these properties based on Poisson approximations and threshold functions for Random Geometric Graphs (RGGs).  We combine these results to develop a new design procedure for computing parameters of robotic swarms that are guaranteed to achieve a specified SCS property.

The paper is structured as follows.  We review related literature in \autoref{Sec:RelWork} and define the SCS properties that we seek to compute and our problem statement in \autoref{Sec:Problem}.  We introduce relevant mathematical concepts in \autoref{Sec:Bground} and provide formal definitions of the SCS properties in \autoref{Sec:Covprops}.   In \autoref{Sec:Geom}, we summarize and extend a computational geometric formulation of SCS, first presented in our work \cite{KumarICRA2014}, that forms the basis of our subsequent computations.  In \autoref{Sec:Unif} and \autoref{Sec:UnifCF}, we derive formulas for SCS properties in cases where robots are distributed uniformly randomly on a boundary.  We develop order-of-magnitude estimates for these formulas in \autoref{Sec:Estim}.  In \autoref{Sec:Design}, we apply our results to compute the number of robots that should be used in a particular boundary coverage scenario to yield desired SCS properties.   \autoref{Sec:NonUnif} extends the analysis of \autoref{Sec:Unif} and \autoref{Sec:UnifCF} to general non-uniform distributions of robots on a boundary. Finally, \autoref{Sec:Conc} concludes the paper and discusses topics for future work.

\section{Related Work}  
\label{Sec:RelWork}
%\todo{Not sure about this}

\subsection{Models of Adsorption Processes}

Our boundary coverage approaches are mechanistically similar to {\it adsorption}, the process of particles binding to a surface for an amount of time that varies with system thermodynamic parameters such as density, temperature, and pressure.   The resulting equilibrium distribution of particles on the surface also varies with these thermodynamic properties. We could emulate \emph{Langmuir adsorption}~\cite{Langmuir1918} by programming the robots to bind (adsorb) to the boundary with some probability and then spontaneously unbind (desorb) after a particular mean residence time.   For instance, Langmuir processes have been used to design nanoparticles that selectively target cell surfaces with high receptor densities \cite{Wang2012}.  However, achieving a target equilibrium robot distribution around a boundary would require strict control over the number of robots and characterization of the thermodynamic variables of the environment (such as tumor tissue).  

Alternatively, we could emulate \emph{random sequential adsorption}~(RSA)~\cite{Evans93, TTVV00} by implementing a Langmuir adsorption process in which unbinding occurs at a much greater time scale than binding.  The resulting robot allocation around the boundary will saturate at some suboptimal packing, at which point there is no room along the boundary for any additional robots.   This approximately irreversible binding of robots to the boundary does not permit control of the equilibrium robot distribution.
R\'{e}nyi showed in \cite{Renyi58} that when particles attach sequentially at random locations along a line without overlap, then the limiting fraction $0.747$ of the line's total length will be occupied by particles.  This fraction, called the \emph{parking constant}~\cite{SolomonWeiner86, Finch03}, defines the maximum degree of boundary coverage that is possible using an RSA strategy. 

\subsection{Stochastic Coverage Strategies for Robotic Swarms}
\label{Subsec:BC}

In our previous work \cite{PavlicISRR2013, PavlicJSDSC14}, we developed a boundary coverage approach for robotic swarms that combines classical RSA with a second irreversible process in which free robots can catalyze the detachment of bound robots.  The robot-catalyzed detachment behavior allows us to achieve any fraction of boundary coverage between 0 and the parking constant.  This coverage is robust to environmental variations such as the number of regions and the swarm size.  In addition, the approach does not require characterizing the rates at which robots encounter boundaries and other robots, which can often be done only through simulation~\cite{Gurarie08}.  In \cite{WilsonSwarmInt2014}, we adapted our  coverage strategy to mimic the behaviors of ants performing collective transport, which we had previously modeled as a stochastic hybrid dynamical system \cite{KumarSHS2013}.  

This prior work focused on designing robot probabilities of attachment and detachment that would achieve target fractions of coverage around multiple boundaries.  In \cite{KumarICRA2014}, we investigated the spatial distributions of robots that attach randomly (without detachment) to a single boundary.  We computed the probability distributions of robot positions along the boundary and distances between adjacent robots. We also derived the probability that a random robot configuration  is {\it saturated}, meaning that each adjacent pair of robots along the boundary lies within a threshold distance, which may represent the diameter of a robot's sensing or communication range.  In the current paper, we build on this prior analysis to develop {\it design procedures} for computing robot parameters that will achieve not only a specified probability of saturation, but also several other statistics related to sensor coverage and communication connectivity around a boundary.

Our focus on stochastic policies is distinct from many previous control approaches to decentralized multi-robot boundary coverage, in which the goal is to drive robots to geometric formations on a circle \cite{Wang2014}.  
However, there is a significant body of work on modeling robotic swarms with stochastic behaviors and controller synthesis for desired collective tasks in such systems \cite{Brambilla2013,CorrellHamann2015}.  Encounter-dependent task allocation strategies are most similar to our stochastic coverage problem, but previous work either deals with scenarios where encountered objects or regions are relatively small (on the scale of one to several robots) \cite{Martinoli04, ref:Labella06, Kanakia2014} or where large objects are covered dynamically by the robots \cite{Correll07}.  In contrast to this work, we address a {\it static} stochastic coverage scenario in which the encountered object or region is {\it large} compared to the robots.  Other related work has addressed the specific problem of optimal mobile sensor deployment along a line with respect to a scalar density field, possibly in the presence of measurement noise \cite{Davison2015Line}, sensor failures \cite{Frasca2015}, and packet loss \cite{Cortes2012}.  

\subsection{Analysis of Graphs and Mobile Networks}
\label{Subsec:Tools}

We use a variety of geometric and probabilistic tools in this paper, which we review in \autoref{Sec:Bground}.  We apply the theory of random graphs, primarily results on thresholds and sharp thresholds from \cite{friezeRandomGraphs}. The graph representing the communication network of a robot team is a \textit{Random Geometric Graph} (RGG), which is discussed extensively in \cite{Matthew2003RGG}. Since many of the formulas for the connectivity, vertex degrees, and components of RGGs are unwieldy, it is more fruitful to determine trends of these properties using Poisson approximation \cite{Matthew2003RGG, barbour1992poisson}. When robots cover a boundary, their communication network can be modeled as a probability density function over a polytope. The quantitative results established in this paper draw from two disparate areas of literature: \textit{order statistics} \cite{ostat2003} and polytope volume computation \cite{DyerFrieze, dyer1991computing}. 

Our work uses results from connectivity analysis of mobile adhoc networks (MANETs), which has a large literature, e.g. \cite{franceschetti2003percolation,xue2004number,hekmat2003degree}. The work in\cite{haenggi2012stochastic} models wireless networks using \textit{Poisson Point Processes} (PPPs).   We use results from \cite{godehardt1996connectivity} in \autoref{Sec:Unif} to count connected components of the  communication network. Our geometric approach differs from the primarily combinatorial one of \cite{godehardt1996connectivity}, although they lead to the same formulas. Combinatorial approaches work well for determining coverage properties when robot positions on the boundary are distributed according to a uniform parent distribution, but they do not extend to non-uniform parents.  Our geometric approach, however, does extend to non-uniform parents. This adaptability comes at the expense of requiring more labor to derive formulas for the graph properties of interest in the uniform case, compared to combinatorial approaches.

\section{Problem Statement}

\label{Sec:Problem}
In this section, we formalize the multi-robot boundary coverage problems that we address. \autoref{Tab:Notation} summarizes the notation used in this paper.   

\begin{table}
\centering
\caption{Notation}
\setlength{\tabcolsep}{1pt}
\begin{tabular}{ll} 
\toprule
\vspace{.5mm}\\
\textbf{Variable} & \textbf{Meaning} \\
\vspace{.5mm}\\
$\upos{i}$ \textbar~$\pos{i}$ & $i$-th Unordered \textbar ~Ordered robot position \\
$\si{i}$ & $i$-th slack \\ 
$\uposvec$ \textbar~$\posvec$ & Unordered \textbar~ Ordered position vector \\
$\svec$ & Slack vector \\
$\psimpn$ \textbar~$\ssimpn$ & Position simplex \textbar~ Slack simplex \\
$\jpos$ \textbar~$\mpos{i}$ & Joint \textbar ~Marginal pdf of ordered positions \\
$\rvpos{i}$ \textbar~$\rvsla{i}$ & Random variables of $i$-th position \textbar~ $i$-th Slack \\
$\jsla$ \textbar~$\msla{i}$ & Joint \textbar ~Marginal pdf of slacks \\
$\hyp{}$ & Hypercuboid \\
$\favn{}$ & Simplex - Hypercuboid intersection region \\
$\vvec$ & Bit vector of the form $\{0,1 \}^n$; e.g. $00110$ \\
$\vm(\mathcal{A})$ & Vertex matrix of polytope $\mathcal{A}$ \\
$\mathbf{e}_i$ & Unit vector along $i$-th coordinate axis \\
\vspace{.5mm} \\
\textbf{Notation} & \textbf{Meaning} \\
\vspace{.5mm} \\
$\mathbf{a}_{i:j}$ & Subvector $(a_i,a_{i+1},\ldots,a_{j})$ of $\mathbf{a}$ \\
$\mathbf{a} \geq \mathbf{b}$ \textbar $\mathbf{a} \leq \mathbf{b}$ & For all $i$, $a_i \geq b_i$ \textbar $a_i \leq b_i$  \\
$\R_+$ & Nonnegative reals \\
$\Vol(f(\mathbf{x}),\Omega)$ & Volume under function $f(\mathbf{x})$ over region $\Omega$ : $\int_{\Omega} f(\mathbf{x}) \mathrm{d}\mathbf{x}$ \\
$\ind_{X>a}$ &  Indicator random variable (rv) for the predicate $X>a$ \\
\vspace{.5mm} \\
\toprule
\end{tabular}
\label{Tab:Notation}
\end{table}

    
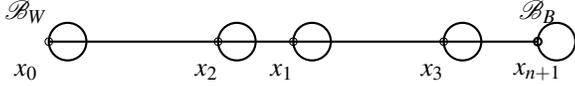
\begin{figure}   
\begin{tikzpicture}

\coordinate(p0) at (-2.25,0);
\coordinate(p4) at (4.25,0);

\coordinate  (a) at (0,0);
\coordinate  (b) at (1,0);
\coordinate  (c) at (3.,0);

\draw[thick] (a) circle (0.25cm);
\node[draw=none,align=left] at (-0.4,-0.4) { $\upos{2}$} ;
\draw[black] (-0.25,0) circle (.05cm); 

\draw[thick] (b) circle (0.25cm);
\node[draw=none,align=left] at (0.6,-0.4) { $\upos{1}$} ;
\draw[black] (0.75,0) circle (.05cm);

\draw[thick] (c) circle (0.25cm);
\node[draw=none,align=left] at (2.6,-0.4) { $\upos{3}$} ;
\draw[black] (2.75,0) circle (.05cm); 

\draw[thick] (p0) circle (0.25cm);
\node[draw=none,align=left] at (-2.8,0.4) { $\Bo_{W}$} ;
\node[draw=none,align=left] at (-2.8,-0.4) { $\upos{0}$} ;

\draw[black] (-2.5,0) circle (.05cm);

\draw[thick] (p4) circle (0.25cm);
\node[draw=none,align=left] at (4.0,-0.4) { $\upos{n+1}$} ;

\node[draw=none,align=left] at (4.0,+0.4) { $\Bo_{B}$} ;
%\draw[black] (5.25,0) circle (.05cm); 

  \draw[thick] (-2.5,0) -- (+4,0);

%\draw[black,fill=white] (-2,0) circle (.1cm); 
\draw[thick] (4,0) circle (.05cm);

%s\draw[dashed,thick,green] (-2,0.24) -- (5,0.24);

\end{tikzpicture}
\caption{Random configuration of robots on a boundary $\Bo$.}
\label{fig:Attach2}
\end{figure}

\end{comment}
\begin{figure}[t]  % htbp
\centering
\begin{tikzpicture}[line width = 1.1pt]
 \draw (0,0) -- (2.5,0);
 \draw (3.5,0) -- (5,0);
 
 \draw (-0.25,0) node[anchor=north] {${ \pos{0}=0}$};
%  \draw (0,0) node[anchor=south] {$0$};
  \draw [fill=black](0,0) circle (.05cm);
 
  %\draw (5,0) node[anchor=south] {$s$};
  \draw (5.5,0) node[anchor=north] {$\pos{n+1}=s$};    
   \draw [fill=black] (5,0) circle (.05cm);
   
     \draw (.5,0) node[anchor=south] {$s_1$};
  \draw (1.0,0) node[anchor=north] {$\pos{1}$};    
    \draw [fill=black](1.0,0) circle (.05cm);
   
     \draw (1.5,0) node[anchor=south] {$s_2$};
  \draw (2.0,0) node[anchor=north] {$\pos{2}$};    
    \draw [fill=black] (2.0,0) circle (.05cm);
   
    \draw  (2.5,0) circle (.05cm);
 \draw  (3.5,0) circle (.05cm);
   \draw [dashed] (2.5,0) -- (3.5,0);
   
   \draw (4.5,0) node[anchor=south] {$s_{n+1}$};
  \draw (4.0,0) node[anchor=north] {$\pos{n}$};    
 \draw [fill=black] (4.0,0) circle (.05cm);
   
\end{tikzpicture}
\caption{Configuration of $\nor$ robot positions $\posvec:= (\pos{1},\ldots,\pos{n})$ on $\Bo$, with artificial robot positions $\pos{0}=0$ and $\pos{n+1}=s$.} 
\label{fig:config}
\end{figure}
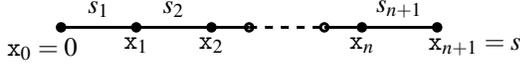
%\end{document}

\subsection{Robot Capabilities}
We consider a group of $\nor$ disk-shaped robots, each with diameter  $\dia$, that are  distributed throughout a bounded environment at random locations. Each robot is equipped with Wi-fi with a communication radius of $d_c$, and it can sense and identify other objects (for instance, using a camera) within a sensing radius of $d_s$.  For computational convenience, we assume that $d_c = d_s$ and label this radius as $\sat$.  The robots have no knowledge of their global positions, nor are they provided with offboard sensing for localization. We  make the simplifying assumption that the robots form a \textit{synchronous network} (i.e. they are synchronized in time with respect to a global clock) and that their controllers do not fail in the course of execution.  

\subsection{Notation and Terminology for Robot Configurations}   % Notation and Terminology for 
%We begin by formalizing the notion of a robot configuration.
The environment contains a line segment boundary $\Bo$, whose endpoints can be distinguished from each other by the robots.  For instance, one endpoint may be white and the other black, which the robots could identify with their cameras.  We label these endpoints as $\Bo_W$ and $\Bo_B$.  Define the unordered position $\upos{i}$ of robot $i$, labeled $\Rob_{i}$, to be that of the robot's endpoint closest to $\Bo_W$, as shown in \autoref{fig:Attach2}. We define $\uposvec = [\upos{1}~...~\upos{\nor}]^T$ as the vector of all the robots' unordered positions.  By sorting the positions in $\uposvec$ according to increasing distance from $\Bo_W$, we obtain the ordered position vector $\posvec = [\pos{i}~...~\pos{\nor}]^T$. We define the distance between $\pos{i}$ and $\pos{i+1}$ as the $i$-th {\it slack}, $\si{i}$. We introduce two artificial robots $\Rob_{0}$ and $\Rob_{\nor+1}$ at unordered positions $\upos{0} = 0$ and $\upos{\nor+1} = s$, respectively, which create the slacks $\si{1} = \pos{1}$ and $\si{\nor+1} = s-\pos{\nor}$. 
We define the {\it slack vector} as $\svec = [\si{1}~...~\si{\nor+1}]^T$.  We say that slack $s_i$ is \textit{connected} if $s_i\leq d$, and \textit{disconnected} otherwise.  \autoref{fig:config} illustrates a configuration of ordered robot positions and slacks.

When a coverage scheme does not permit robots to physically overlap along the boundary, we refer to it as \textit{conflict-free (CF)}; otherwise, we call it \textit{conflict-tolerant (CT)}. Although CT schemes are not realistic for rigid-body robots that move in the plane,  they are easier to analyze than their CF equivalents, and quantitative results on properties of CT schemes can be readily extended to CF schemes.

We note that all our analysis can be adapted to boundaries that are closed curves, 
as discussed in our previous work \cite{KumarICRA2014}.

\subsection{System Design for Desired Robot Configuration Properties}

We now introduce terminology pertaining to the communication and sensing networks that are formed by a random robot configuration on $\Bo$.  The {\it communication graph} $\G$ of a robot configuration is a graph whose vertices are the robot positions $\pos{i:1,\ldots,n}$ (excluding the artificial robot positions) and whose edges are defined as $(\pos{i},\pos{j})$ iff $|\pos{j}-\pos{i}| \leq \con$.  The {\it vertex degree} ($\dgr$) of a robot is the number of neighboring robots along $\Bo$ that are  within its communication range.  This property can be used to estimate the number of robots that can detach from $\Bo$ without $\G$ losing connectivity, and hence measures the robustness of the network to node deletion and failure.  We define $\cmp$ to be the \textit{number of connected components} of $\G$.  This quantity can be used to estimate the number of additional robots that are required to make $\G$ {\it connected}, which occurs when $\cmp = 1$.  The \textit{sensed length} ($\slen$) of $\Bo$ is the total length of the subset of $\Bo$ that is sensed by at least one robot. We say that $\Bo$ is \textit{fully sensed} iff every point on it is sensed.  A robot configuration \textit{monitors} $\Bo$ iff it is connected \textit{and} senses the entirety of $\Bo$.  We define $\pcon$ as the probability that a random robot configuration on $\Bo$ is connected and $\pmon$ as the probability that a configuration monitors $\Bo$.

In this paper, we derive analytical expressions or approximations of the following five properties of a robot configuration generated by an SCS: $\pmon$, $\pcon$, and the expected values of $\slen$, $\cmp$, and $\dgr$.  As we will show, these properties are functions of $n$, $D$, $d$, and $s$.  We will present a procedure for computing one of these four parameters, given values for the other three, that will generate random robot configurations with a desired value of one of the five properties.

\section{Mathematical Preliminaries}
\label{Sec:Bground}

We now introduce a number of concepts that will be used throughout the paper. 

\subsection{Random Geometric Graphs (RGG)}
\label{Subsec:rgg}
Let $\rvuposvec_{1:n}$ be a vector of i.i.d. random variables (rv's), and let $\upos{i}$ be a realization of $\rvupos{i}$.  
Define a \textit{Random Geometric Graph} (RGG), denoted by $\G = \G(n,\con)$, with vertices $\{\uposvec_{1:n}\}$ and edges consisting of vertex pairs $(\upos{i},\upos{j})$ for which  $||\upos{i} - \upos{j}|| \leq \con$. The  properties of general RGGs, including their number of clusters, their probability of connectivity, and the size of their largest connected component, are difficult to compute precisely, but their asymptotics have been studied extensively in \cite{Matthew2003RGG}. We will derive exact formulas for the properties of $\G$ where possible, and  estimate  these properties otherwise.

We use the following definitions from \cite{friezeRandomGraphs}. Consider the property that $\G$ is connected. Since this property remains true if a random edge is added to $\G$, it is said to be \textit{monotone increasing}. Likewise, the property that $\G$ has at least $k$ components is {\it monotone decreasing}, since it remains true if a random edge is removed.  

Let $\G$ be an RGG with $n$ vertices and $m(n)$ edges, and let $\Prob(\G \in \mathscr{P})$ denote the probability that $\G$ has the monotone increasing property $\mathscr{P}$.  A function $m^*(n)$ that satisfies the following conditions is called a \textit{threshold function} for $\mathscr{P}$:
\begin{enumerate}
 \item If $m(n) = o(m^*(n))$, then $\G \notin \mathscr{P}$ \textit{almost surely} (a.s.)
 \item If $m(n) = \Omega(m^*(n))$, then $\G \in \mathscr{P}$ a.s.
\end{enumerate}
A threshold function $m^*(n)$ for a monotone decreasing $\mathscr{P}$ is defined exactly as above, but with $o$ and $\Omega$ switched. The threshold function $m^*(n)$ is called a \textit{sharp threshold} for a monotonically increasing property iff the following conditions hold for all $\epsilon >0$ \cite{friezeRandomGraphs}: 
\begin{enumerate}
 \item if $\frac{m(n)}{ m^*(n)} \leq 1-\epsilon$, then $\G \notin \mathscr{P}$ a.s., and  
 \item if $\frac{m(n)}{ m^*(n)} \geq 1+\epsilon$, then $\G \in \mathscr{P}$ a.s. 
\end{enumerate}

\vspace{2mm}

\begin{theorem}
\label{thm:sharpth}
 Every monotone property of an RGG has a sharp threshold \cite{goel2004sharp}. 
\end{theorem}

\subsection{Geometric Objects}   % construction, properties, operations
\label{Subsec:polytopes}

\subsubsection{Polytopes}
A geometric shape $\genpoly$ embedded in $\R^\nor $ is called a \textit{polytope} if it is bounded on all sides by hyperplanes; it is \textit{convex} if it can be expressed as an intersection of half-spaces \cite[Ch. 1]{ConvPoly}. A convex polytope that is specified by a collection of half-space inequalities is said to be in $H$ (hyperplane) form.  Alternatively, it may be defined in the $V$ (vertex) form as the convex hull of a set of vertices. %, which are input. 
We will commonly describe a polytope using its \textit{vertex matrix}, $\vm(\genpoly)$, each of whose columns gives the coordinates of one of the vertices of $\genpoly$.  Software routines for converting $\genpoly$ from one form to another to determine its convex hull and find its volume $\Vol(\genpoly)$ are available in computational geometric packages such as Sage \cite{sage} and cddlib \cite{cddlib}.

A polytope $\gensimp$ with $\nor +1$ vertices embedded in $\R^n$ is called a \textit{simplex}. The \textit{canonical simplex} is defined by 
\begin{align}
 \Delta_n = \{ \mathbf{x} \in \R^n: \1^T \mathbf{x} \leq 1 ~\text{and}~ \mathbf{x} \geq \0 \}
\end{align}
in its $H$ form, and its vertex matrix is the identity matrix $\eye_{n+1}$. Every simplex in $\R^n$ is a linear transformation of $\Delta_n$. The volume of a simplex may be determined in polynomial time from its $V$ form by a determinant \cite{DyerFrieze}.

\subsubsection{ Simplex-Hypercube Intersection}
\label{Subsec:genshi}

The following result on computing the volume of intersection between a  half-space and the unit hypercube is quoted from \cite{DyerFrieze, dyer1991computing}. Define a positive half-space by
 \begin{align}
 \label{eqn:hsp}
 \mathcal{T} := \{\svec \in \R^n: \mathbf{a}^T \mathbf{s} \leq b \}, ~\text{where} ~ \mathbf{a} >\0 ~\text{and} ~ b>0.
 \end{align}
Let $\favn_{gen}$ be the intersection of $\mathcal{T}$ with the unit hypercube $\hyp{cube} := [0,1]^n$. For every vertex $\vvec$ of $\hyp{cube}$, define the simplex $\Delta(\vvec)$ by 
\begin{align}
 \Delta(\vvec) = \{\svec \in \mathcal{T} : \svec \geq \vvec \} \cap \hyp{cube}.
\end{align}
The vertices of this simplex are $\vvec$ and the points $\mathbf{p}_i = \frac{1}{a_i}(b-\avec^T \vvec) \mathbf{e}_i$, where each $\mathbf{e}_{i:1,\ldots,n}$ is a unit vector along the $i$-th axis. Let $\vvec \in \{0,1\}^n$ denote an $n$-bit vector (i.e. a vector with $n$ entries that are zero or one) that iterates through the vertices of $\hyp{cube}$. 

\vspace{2mm}

\begin{lemma}
\label{thm:vhsp}
The volume of $\favn_{gen} = \mathcal{T} \cap \hyp{cube}$ is: 
\begin{eqnarray}
\label{eqn:vhsp}
\hspace{-4mm} \Vol(\favn_{gen}) \hspace{-2mm} &=& \hspace{-2mm} \sum _{\vvec \in \{0,1 \}^n } (-1)^{\1^T \vvec} \hspace{1mm}\Vol(\Delta(\vvec)) \nonumber \\ 
 &=& \hspace{-2mm} \frac{1}{n! \prod {a_i}}  \sum _{\vvec \in \{0,1\}^n} (-1)^ {\1^T \vvec} \left(\max( b-\avec ^T \mathbf{v},0)\right)^{n}. 
\end{eqnarray}
\end{lemma}
When $\vvec$ lies in $\mathcal{T}$, we have $(b-\avec^T\vvec) \geq 0$, and the resulting simplex $\Delta(\vvec)$ contributes a nonzero volume to $\favn_{gen}$. Otherwise, the term $b-\avec^T \vvec$ is negative, and the resulting $\Delta(\vvec)$ contributes nothing to $\Vol(\favn_{gen})$. The sum in \autoref{eqn:vhsp} has alternating positive and negative terms that arise from the implicit application of the Principle of Inclusion and Exclusion (PIE) \cite{vanLintCombin2001}. We will later express volumes of polytope intersections as sums of terms with alternating opposite signs, similar to \autoref{eqn:vhsp}. Note that the subset of $\mathcal{T}$ lying in the positive orthant of $\R^n$, 
\begin{align}
\label{eqn:tsimp}
\mathcal{T}_{simp} = \{ \svec \in \R_+^n : \svec \geq 0 ~\text{and}~  \mathbf{a}^T \mathbf{s} \leq b \}, 
\end{align}
defines a simplex with vertices at $\0$ and at the $n$ points $\frac{b}{a_i} \mathbf{e}_i$, where  $\mathcal{T}$ intercepts each  coordinate axis. Hence, we can express $\favn_{gen}$ as the simplex-hypercube intersection $\mathcal{T}_{simp} \cap \hyp{cube}$, a fact which will be exploited in later sections. Finally, we note that \autoref{eqn:vhsp} takes $\Omega(2^n)$ time to evaluate.  

\subsection{Probability Theory and Statistics}  
\label{Subsec:prob}
We write $X \sim f_X(x): x\in \R$ to indicate that $X$ is a real-valued rv with probability density function (pdf) $f_X(x)$. Similarly, we write $\mathbf{X} \sim f_\mathbf{X}(\mathbf{x}):\mathbf{x}\in \R^n$ to  indicate that $\mathbf{X}$ is a real-valued random vector with pdf $f_\mathbf{X}(\mathbf{x})$.  We will use $ F_X(x)$ to denote the cumulative distribution function (CDF) associated with $f_X(x)$. 

Let $X$ be a real-valued rv defined as above. We use $\ind_{A}$ to denote the \textit{indicator function} that is defined to be unity over the region $A$ in its subscript, and zero elsewhere. An \textit{indicator rv} such as $\ind_{X\geq 1}$ is one which is unity if the event in its subscript occurs, and zero otherwise. Let $f_\mathbf{X}(\mathbf{x})$ be the joint pdf of the variables ${X}_{1:n}$, whose support is a region $\Omega \in \R^n$. If $\Omega'$ is a subset of $\Omega$, then the measure of $\Omega'$ induced by $f$ is $\int_{\Omega'} f_{\mathbf{X}}(\mathbf{x}) d\mathbf{x}$. Since this integral gives the volume under the pdf $f_{\mathbf{X}}$ over $\Omega'$, we will denote it by $\Vol(f_{\mathbf{X}} , \Omega')$.

\subsubsection{Poissonization}
\label{Subsubsec:poisson}
Many of our later computations will involve sums such as $X_1+\ldots+X_n$, where the $X_i$ are \textit{dependent} rv's. Because of this dependence, the strong laws of large numbers cannot be directly applied to this sum.  Instead, this sum can be approximated by a Poisson rv through a process called {\it Poissonization}, which we will use in \autoref{Sec:Estim}.

We say that the rv's $X_{i:1,\ldots,n}$, are \textit{negatively associated} (n.a.) if an increase (resp., decrease) in one of the rv's causes the others to decrease (resp., increase); the precise definition is given in \cite[p. 26]{barbour1992poisson}. Suppose that we have n.a. indicator rv's  $\ind_{X_{1:n}}$ and let $W:=\sum_{i=1}^n \ind_{X_i}$. We may then approximate  $f_W(w)$ with a Poisson rv with the pdf $\Poi(\lambda)$, where $\lambda := \Exp(W) =  \sum_{i=1}^n \Exp(\ind_{X_i})$, which is justified as follows. First,  we define the \textit{total variation (TV) distance} $d_{TV}(f,g)$ between two probability measures $f$ and $g$ on a sample space $\mathscr{\omega}$ equipped with a sigma algebra $\sigma$.  The TV distance is the maximum possible difference between the probabilities assigned by $f$ and $g$ to the same event:
\begin{align}
 d_{TV} (f,g) = \sup _{E \in \sigma} || f(E) - g(E) ||.
\end{align}
Then we can apply the following result, which implies that the TV distance between the pdf of $W$ and  $\Poi(\lambda)$ diminishes exponentially as $\lambda \rightarrow \infty$.  
\vspace{2mm}

\begin{lemma}
\label{thm:poisson}
 From \cite{barbour1992poisson},
 \begin{align}
  d_{TV} ( f_W(w), \mathrm{Poi}(\lambda) ) ~\leq~ (1-\exp(-\lambda)) \left(1- \frac{\Var (W)}{\lambda}\right),
  \end{align}
\end{lemma}
where $\Var (W)$ is the variance of $W$.

\subsubsection{Order Statistics}
\label{Subsec:ostat}
Suppose the rv's   $\rvupos{i}$ in the vector $\rvuposvec_{1:n}$ in \autoref{Subsec:rgg} are drawn from the same pdf $\fpar(x):x\in \R$. This pdf is called the \textit{parent pdf} (or just \textit{parent}, when the context is clear) of the rv's.   The ordered rv  $\rvpos{i}$  is called the $i$-th \textit{order statistic} of the parent variable $\rvupos{i}$ and is the $i$-th smallest of the $\nor$ components of $\rvuposvec$. The joint pdf of the $\nor$ order statistics and the marginal pdf of the $i$-th order statistic $\rvpos{i}$ are given by \cite{ostat2003}
\begin{eqnarray}
\label{eqn:ostat}
\jpos  &=& \nor ! \prod_{i=1}^{\nor} {\fpar(t)} \ind_{\pos{1}\leq \pos{2} \ldots \pos{n-1} \leq \pos{n}},  \nonumber \\
\mpos{i}  &=& \sum_{j=i}^{\nor} \binom{\nor}{j} {\Fpar(t)^{j}} (1-\Fpar(t))^{n-j},
\end{eqnarray}
where $G$ is the parent CDF.

\section{Definition of Coverage Properties}
\label{Sec:Covprops}

In this section, we give general definitions of the coverage properties of interest in terms of indicator rv's.  These formulas will be applied in Sections \ref{Sec:Unif}, \ref{Sec:UnifCF}, \ref{Sec:Estim}, and \ref{Sec:NonUnif} to derive the properties for different coverage scenarios.

\subsection{Definitions of events $\mathrm{con}$, $\sen$, and $\mon$}
\label{Subsec:covprops-con}
Let $\mathrm{con}$ be the event that a robot configuration on $\Bo$ is \textbf{connected}.  This event can be expressed as 
\begin{equation}
\label{eq:condef}
\mathrm{con} := \prod_{i=2}^{n} \ind_{S_i\leq \con}.
\end{equation}
The event that $\Bo$ is \textbf{fully sensed} is:
\begin{align}
\label{eq:sendef}
\sen :=  \ind_{S_1 \leq \con}  \ind_{S_{n+1} \leq \con} \cdot  \prod_{i=2}^{n}  \ind_{S_i \leq 2\con}.
\end{align}
Likewise, the event that $\Bo$ is \textbf{monitored} is:
\begin{equation}
\label{eq:mondef}
\mathrm{\mon} := \prod_{i=1}^{n+1} \ind_{S_i \leq \con}.
\end{equation}

\subsection{Numbers of connected, fully sensed, and monitored slacks}
\label{Subsec:covprops-ncon}
The events $\mathrm{con}$, $\sen$, and $\mon$ are all products of indicator functions, which cannot be Poissonized using Lemma \autoref{thm:poisson}. On the other hand, the following functions, defined as sums of indicators, can be Poissonized:
\begin{eqnarray}
 \Ncon &:=& \sum_{i=2}^{n} \ind_{S_i\leq \con},  \label{eq:Ncon} \\
 \Nsen &:=& \sum_{i=1,n+1} \ind_{S_i \leq \con} + \sum_{i=2}^{n} \ind_{S_i\leq 2 \con},  \label{eq:Nsen}\\  
  \Nmon &:=& \sum_{i=1}^{n+1} \ind_{S_i \leq \con}. \label{eq:Nmon}
\end{eqnarray}
Here, $\Ncon$, $\Nsen$, and $\Nmon$ are the numbers of slacks that are connected, fully sensed, and monitored, respectively, by a robot configuration. 

\subsection{Sensed Length, $\slen$}
\label{Subsec:Covprops-slen}
 We compute the \textbf{sensed length} $\slen$ of a robot configuration by summing the lengths of $\Bo$ that are sensed by the robots on each slack.   The first robot $\pos{1}$ senses  length $\min(s_1,d)$ on the slack $s_1$. Every slack $s_{i:2,\ldots,n}$ will be sensed by the two robots at $\pos{i}$ and $\pos{i+1}$, which will together sense a length of $\min(s_i,2d)$ on it. Likewise, the last robot $\pos{n}$ senses length $\min(s_{n+1},d)$ on $\si{n+1}$. The total sensed length is thus
 \begin{align}
  \slen := \sum_{i=1,n+1} \min(s_i,d) + \sum_{i=2}^{n} \min(s_i,2d).
 \end{align}
Defining $\psen$ as the probability of the event $\sen$ in \autoref{eq:sendef}, we may express the expectation of $\slen$ as
 \begin{align}
 \label{eqn:expslen}
\Exp(\slen) = \psen \cdot s.
 \end{align}

\subsection{Number of Connected Components, $\cmp$}
\label{Subsec:Covprops-cmp}
Using the fact that each disconnected slack increments the number of \textbf{connected components ($\cmp$)} in $\G$, we may write 
\begin{align}
\label{eqn:cmpnum}
 \cmp = 1+ \sum_{i=2}^{n} \ind_{S_i > \con},
\end{align}
Applying the linearity of expectation to \autoref{eqn:cmpnum}, we find the expected value of $\cmp$ to be 
\begin{equation}
\label{eq:expcmp}
\Exp(\cmp) = 1+ \sum_{i=2}^{n} \Exp(\ind_{S_i > \con}) = 1+ \sum_{i=2}^{n} \Prob(S_i > \con).
\end{equation}

\subsection{Vertex Degree, $\dgr$}
\label{Subsec:Covprops-dgr}
The \textbf{vertex degree ($\dgr$)} of a robot at position $\pos{i}$ is 
\begin{align}
\label{eqn:degvert}
 \deg(\pos{i}) = \sum_{1\leq j\leq n, j\neq i} \ind_{|\pos{i} - \pos{j}| \leq d}.
\end{align}
Applying the linearity of expectation to \autoref{eqn:degvert}, the expected value of $\dgr$ is 
\begin{equation}
\label{eq:expdgr1}
\Exp(\dgr) = (n-1) \Prob(|X_i - X_j| \leq \con),
\end{equation}
where $X_i$ and $X_j$ are any two unordered positions on $\Bo$. 

\end{comment}

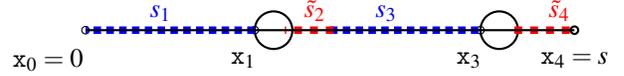
\begin{figure}
   \begin{tikzpicture}

\coordinate(p0) at (-2.25,0);
\coordinate(p4) at (4.25,0);

\coordinate  (a) at (0,0);
\coordinate  (b) at (1,0);
\coordinate  (c) at (3.,0);

\draw[thick] (a) circle (0.25cm);
\node[draw=none,align=left] at (-0.4,-0.4) { $\pos{1}$} ;
\draw[black] (-0.25,0) circle (.05cm);

\draw[thick] (c) circle (0.25cm);
\node[draw=none,align=left] at (2.6,-0.4) { $\pos{3}$} ;
\draw[black] (2.75,0) circle (.05cm); 

\node[draw=none,align=left] at (-3.0,-0.4) { $\pos{0}=0$} ;
\draw[black] (-2.5,0) circle (.05cm); 

\node[draw=none,align=left] at (4.0,-0.4) { $\pos{4}=s$} ;
%\draw[black] (5.25,0) circle (.05cm); 

     \draw[color=blue] (-1.5,0) node[anchor=south] {$\si{1}$};
     \draw[color=blue][dotted][line width=.1cm] (-.25,0) -- (-2.5,0);

     \draw[color=blue] (1.5,0) node[anchor=south] {$\si{3}$};
     \draw[color=blue][dotted][line width=.1cm] (0.75,0) -- (2.75,0);

      \draw[color=red] (0.55,0) node[anchor=south] {$\fsi{2}$};
     \draw[color=red][dashed][line width=.1cm] (0.8,0) -- (0.15,0);

         \draw[color=red] (3.8,0) node[anchor=south] {$\fsi{4}$};
     \draw[color=red][dashed][line width=.1cm] (3.2,0) -- (4,0);

\draw[thick] (-2.5,0) -- (+4,0);

%\draw[black,fill=white] (-2,0) circle (.1cm); 
\draw[thick] (4,0) circle (.05cm);

%s\draw[dashed,thick,green] (-2,0.24) -- (5,0.24);

\end{tikzpicture}
\caption{Slacks (blue) and free slacks (red) of the robot configuration in \autoref{fig:Attach2}.  }
\label{fig:CovProps}
\end{figure}

\section{Geometric Formulation}
\label{Sec:Geom}

We now present a geometric formulation of robot positions and slacks as points in high-dimensional spaces. This geometric approach will be used in \autoref{Sec:Unif} through \autoref{Sec:NonUnif} to compute the properties given in \autoref{Sec:Covprops} for different problem scenarios. 

Consider the vector of robot positions $\posvec$ as a point in $\R^n$, and neglect CF requirements.  Every valid position vector on $\Bo$ satisfies the constraints 
\begin{align}
 \label{eqn:pconstraint}
 0\leq \pos{1} \leq \ldots \leq \pos{n} \leq s. 
\end{align}
\autoref{eqn:pconstraint} gives the $H$ form of a simplex in $\R^n$, which we refer to as the {\it position simplex} $\psimpn$. We could alternatively consider the slack vector defined by $\posvec$,
\begin{align}
\label{eqn:svecdef}
 \svec_{1:n+1} : = \posvec_{1:n+1} - \posvec_{0:n}.
\end{align}
Representing $\svec$ as a point in $\R^{n+1}$, we observe that a valid slack vector satisfies the constraints
\begin{equation}
\label{eqn:sconstraint}
\0 \leq \svec ~~~\text{and}~~~ \sum_{i=1}^{n+1} \si{i} =  \1^T \svec = s.
%\0 \leq \svec ~\text{and}~ \1^T \svec \leq s. 
\end{equation}
% $\ssimpn$ loses one dimension, since 
These inequalities define the $H$ form of a \textit{degenerate} $n$-dimensional simplex $\ssimpn$ embedded in $\R^{n+1}$, since $\svec$ can be completely specified by $\nor$ slacks instead of $\nor+1$. By dropping $\si{n+1}$, we may redefine $\ssimpn$ as 
\begin{align}
 \label{eqn:sconstraint2}
 \ssimpnf :=\{ \svec \in \R^n : \0 \leq \svec ~ \text{and}~ \1^T \svec \leq s \},
\end{align}
which is a full-dimensional irregular simplex in $\R^n$. \autoref{eqn:sconstraint} defines the regular simplex $s \cdot \Delta_n$ in $\R^{n+1}$, which scales the canonical simplex $\Delta_n$ by a factor of $s$. This regularity often simplifies connectivity-related computations. The full-dimensional form $\ssimpnf$  represents the last slack $\si{n+1}$ implicitly, and so constraints on $\si{n+1}$ must be treated separately. It is easier to use software libraries to compute volumes and integrals
over $\ssimpnf$ than over the degenerate form $\ssimpn$. We shall prefer one representation over the other depending upon convenience. \autoref{fig:P2} illustrates $\psimpn$ and $\ssimpn$ for  $n=2$ robots.

\end{comment}
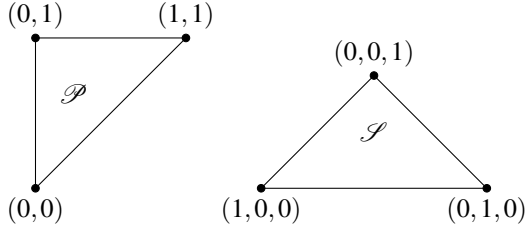
\begin{figure}[t] % !htbp
\centering
\begin{tikzpicture}
 \draw (0,0) -- (0,2);
 \draw (0,0) -- (2,2);
 \draw (0,2) -- (2,2); 

 \draw [fill=black] (0,-0.) circle (0.05) node[below, black]{$(0,0)$};
  \draw [fill=black] (0,2.) circle (0.05) node[above, black]{$(0,1)$};
 \draw [fill=black] (2,2.) circle (0.05) node[above, black]{$(1,1)$};

  \draw (.5,1.5) node[anchor=north] {$\mathcal{P}$};

  %%%%%%%%%%%%%%%%%%%%%%%
   \draw (3,0) -- (6,0);
 \draw (3,0) -- (4.5,1.5);
 \draw (6,0) -- (4.5,1.5); 

 \draw [fill=black] (3,-0.) circle (0.05) node[below, black]{${\small (1,0,0)}$};
  \draw [fill=black] (6,-0.) circle (0.05) node[below, black]{$(0,1,0)$};
 \draw [fill=black] (4.5,1.5) circle (0.05) node[above, black]{$(0,0,1)$};
 
  \draw (4.5,1) node[anchor=north] {$\mathcal{S}$};

 \end{tikzpicture}
  \caption{Position simplex and slack simplex for $n=2$ robots and $\sla=1$.}
  \label{fig:P2}
\end{figure}
 % \end{document}

We now describe how CF coverage, connectivity, and the condition of monitoring can be expressed in terms of constraints over $\psimpn$ or $\ssimpn$.

\paragraph{CF Constraints}
CF constraints prohibit robots of diameter $\dia$ from extending beyond $\Bo$ and from conflicting (i.e., overlapping).  These constraints are defined as 
\begin{eqnarray}
\label{eqn:cfconstraints}
 \pos{i} - \pos{i-1} &\geq& \dia, ~~~ i=1,\ldots,n+1.
\end{eqnarray}

Note that \autoref{eqn:cfconstraints} enforces $\pos{1} \geq \dia$ and $\pos{n} \leq s-\dia$ to ensure against conflicts with the virtual robots at $\pos{0} =0$ and $\pos{n+1}=s$. The conflict-free subset of $\psimpn$ (and $\ssimpn$) satisfies \autoref{eqn:pconstraint} and \autoref{eqn:cfconstraints}.

\paragraph{Connectivity and Monitoring Constraints}
For a robot configuration to be connected, the slack vector must satisfy the constraint \begin{align}
\label{eqn:connconstraints}
\svec_{2:n} \leq d\1^T. 
\end{align}
For a configuration to monitor $\G$, we require in addition that the slacks $\si{1}$ and $\si{n+1}$ be connected, leading to 
\begin{align}
\label{eqn:monconstraints}
 \svec_{1:n+1} \leq d\1^T.
 \end{align}

Monitoring requires $\svec$ to lie within the $(n+1)$-dimensional hypercube $\hyp{\rm mon} = [0,d]^{n+1}$. On the other hand, connectivity places no constraints on the extremal slacks, which may lie in $[0,s]$; the resulting $\svec$ lies within the hypercuboid $\hyp{\rm con} = [0,s] \times [0,d]^{n-1} \times [0,s]$.  Note that a minimum of $\nmin := \lfloor s/d \rfloor $ robots is required to monitor $\Bo$.

\section{CT Coverage with Uniform I.I.D. Parents} 
\label{Sec:Unif}

In this section, we analyze conflict-tolerant (CT) coverage schemes with a uniform i.i.d. parent pdf.    Since the uniform measure of a polytope is proportional to its volume, the probability of an event $\favn \subseteq \ssimpn$  can be computed as $\Vol(\favn) / \Vol(\ssimpn)$, a fact that does not hold for general i.i.d. parents.    Consequently, this case will generally involve the simplest computations.

\subsection{Spatial pdfs of Robot Positions and Slacks}
\label{Subsec:spatial}

The uniform parent is defined by $\fpar := \frac{1}{s} \ind_{\Bo}$.  From \autoref{eqn:ostat}, the joint pdf of robot positions is uniform over $\psimpn$, and $\rvpos{i}$ has a Beta pdf:
\begin{eqnarray} %\label{eqn:ostat-uni}
 \jpos &=& \frac{n!}{s^n} \ind_\psimpn \label{eqn:unijpos} \\
 f_{\rvpos{i}} (\pos{i}) &=& s \cdot \mathrm{Beta}(i,n-i+1). \label{eqn:unimpos}
\end{eqnarray}

Making the change of variables  $\rvsla{i} = \rvpos{i} - \rvpos{i-1}$ in \autoref{eqn:unijpos}, we find that $\jsla$ is uniform over the slack simplex $\ssimpn$:
\begin{align}
\label{eqn:ostat-unisla}
 \jsla \sim \frac{n!}{s^{n+1}} \ind_{\ssimpn}.
\end{align}
Since the slacks are not subject to ordering constraints, they are {\it exchangeable} rv's and thus are identically distributed. In particular, $\rvsla{i} \sim \rvsla{1}$, and since $\rvsla{1} = \rvpos{1}$, we have by \autoref{eqn:unimpos} that
\begin{align}
\label{eqn:msla}
 f_{S_i} (s_i) = s \cdot {\mathrm{Beta}(1,n)}.
\end{align}

The mean positions at $\Exp(\rvpos{i})=  \frac{s \cdot i}{\nor+1}$ subdivide $\Bo$ into $\nor+1$ equal slacks, each of length $\Exp(\rvsla{i}) = \frac{s}{\nor+1}$, as we would expect of robot configurations on average for a uniform parent pdf.

As $\nor$ and $\con$ increase, we expect $\pmon$, $\pcon$, $\slen$, and $\dgr$ to increase monotonically regardless of the parent pdf $\fpar$. The property $\cmp$ decreases monotonically with $\con$, but as we will see in \autoref{Sec:Design}, it does not vary monotonically with $\nor$.   We will now derive each of these graph properties for the uniform parent.

\subsection{Probability of Monitoring, $\pmon$} 
\label{Subsec:pmon}

From \autoref{eqn:monconstraints}, the subset of $\ssimpn$ whose configurations monitor $\Bo$ is $\favn_{\mon} := \ssimpn \cap \hyp{\mon}$, which we refer to as the \textit{favorable region} for monitoring. Since $\favn_{\mon}$ is the intersection of two convex polytopes, it is a convex polytope as well. We then have $\pmon = \Vol(\favn_{\mon})/\Vol(\ssimpn)$.  While we may determine $\Vol(\favn_{\mon})$ by first triangulating it into simplices, in practice this procedure is computationally intensive. Instead, we compute this volume using an approach that we developed in \cite{KumarICRA2014}, which we summarize here.

 Let $\unfavn_{\mon} : =  \ssimpn \setminus \favn_{\mon}$ be the exterior of $\favn_{\mon}$. This region consists of all slack vectors with at least one disconnected slack. We will express $\unfavn_{\mon}$ as the union of \textit{intersecting} simplices. Consider that subset of $\unfavn_{\mon}$ which has slack vectors with a single disconnected slack $s_k$. Define $s'_k = s_k-d$, and note that all such slack vectors satisfy 
\begin{align}
\label{eqn:onediscslack}
s'_k ~+ \sum_{1\leq j \leq n+1,j\neq k} s_j  ~=~ s-d, 
\end{align}
forming a regular simplex of side $\sqrt{2}(s-d)$. We will call this simplex the \textit{exterior simplex} of $s_k$ and denote it by $\unfavn_{\mon}(s_k>d)$. It is clear that 
\begin{align}
 \unfavn_{\mon} = \bigcup _{k=1}^{n+1} \unfavn_{\mon}(s_k >d).
\end{align}
However, because the simplices $\unfavn_{\mon}(s_k>d)$ overlap, % so that 
\begin{align}
  \sum _{k=1}^{n+1} \Vol( \unfavn_{\mon}(s_k >d)) > \Vol(\unfavn_{\mon}).  \nonumber 
\end{align}
Consequently, finding $\Vol(\unfavn_{\mon})$ requires applying the PIE. To do so, let $\vvec \in \{0,1\}^{n+1}$, and define its associated exterior simplex $\unfavn_{\mon}(\vvec)$ by 
\begin{align}
\label{eqn:uvvec}
 \unfavn_{\mon}(\vvec) := \{ \svec \in \unfavn_{\mon}: \svec \geq d \vvec \}.
\end{align}
Every 1-bit component $v_i=1$ in $\vvec$ causes the corresponding slack $s_i$   to be disconnected in the exterior simplex; other slacks, associated with the 0-bits, are \textit{unrestricted}.  
Analogously to \autoref{eqn:onediscslack},  $\unfavn_{\mon}(\vvec)$ is a regular simplex of side $\sqrt{2} (s-kd)$, and its volume is 
\begin{equation} 
\Vol(\unfavn_{\mon}(\vvec)) = \frac{(s-kd)^n\sqrt{n+1}}{n!}.
\end{equation} 

In \autoref{eqn:uvvec}, we need only consider those vectors $\vvec$ with at most $\nmin$ bits set to 1,  since a larger number of disconnected slacks would cause $\svec$ to fall outside $\ssimpn$. Applying the PIE, we compute
\begin{eqnarray}
\label{eqn:unfavnmon}
 \Vol(\unfavn_{\mon}) &=& \sum _{\vvec \in \{0,1\}^{n+1}:~ 1\leq \1^T \vvec \leq \nmin } \Vol(\unfavn_{\mon}(\vvec)) \nonumber \\
  &=& \sum _{k=1}^{\nmin} (-1)^{k-1} \binom{n+1}{k} (s-kd) ^{n} \frac{\sqrt{n+1}}{n!}. \nonumber \\
\end{eqnarray}
Finally, we obtain an expression for $\pmon$: 
\begin{eqnarray}
\label{eqn:pmon}
 \pmon &=& \frac{\Vol(\favn_{\mon})}{\Vol(\ssimpn)} ~=~ 1- \frac{\Vol(\unfavn_{\mon})} {\Vol(\ssimpn)} \nonumber \\
 &=&   1 - \sum _{k=1}^{\nmin} (-1)^{k-1} \binom{n+1}{k} \left(1-k\frac{d}{s}\right) ^{n}.
\end{eqnarray}

Alternatively, we can use Lemma \autoref{thm:vhsp} to compute $\Vol(\favn_{\mon})$ as follows.   We define $\favn_{\mon}$ in its full-dimensional form as the intersection between the cube $\hyp{cube} = [0,d]^n$ and \textit{two} regions:
\begin{align}
\label{eqn:2inter}
 \favn_{\rm full, \mon} = \ssimpnf \cap   {\ssimpnf}(s_{n+1}\leq d) \cap \hyp{cube}
\end{align}
Here, the region 
\begin{align}
\label{eqn:np1conn}
 {\ssimpnf}(s_{n+1}\leq d) := \{\svec_{1:n} \in \ssimpnf : 0\leq (s-\1^T \svec_{1:n}) \leq d \}
\end{align}
captures the connectivity of $\si{n+1}$. 
To compute $\Vol(\favn_{\rm full, \mon})$, we first define the region $\ssimpnf(s_{n+1} > d)$ to be the subset of $\ssimpnf$ over which $\si{n+1}$ is disconnected. We note that 
\begin{align}
 \favn_{\rm full, \mon} =  (\ssimpnf \setminus   {\ssimpnf}(s_{n+1} > d)) \cap \hyp{cube};
\end{align}
exploiting the fact that $\ssimpnf(s_{n+1} > d)$ is a simplex, we obtain
\begin{align}
\label{eqn:2intervol}
\Vol(\favn_{\rm full, \mon}) &= \Vol(\ssimpnf \cap \hyp{cube}) \nonumber \\
  & ~~~~ - \Vol({\ssimpnf}(s_{n+1} > d) \cap \hyp{cube})
\end{align}
by applying Lemma \autoref{thm:vhsp} once for each volume of intersection.

\subsection{Probability of Connectivity, $\pcon$} 
\label{Subsec:pcon}
From \autoref{eqn:connconstraints}, connectivity of the graph $\G$ does not place any constraints on  $\si{1}$ and $\si{n+1}$. Instead, we define the connected region $\favn_{\text{con}}$  to be the subset of $\ssimpn$ consisting of connected slack vectors, analogous to $\favn_{\mon}$. We may write $\favn_{\text{con}}$ in full-dimensional form as $\favn_{\text{con}} := \ssimpn_{\rm full} \cap \hyp{\text{con}}$, where $\hyp{\text{con}}$ is defined in \autoref{Sec:Geom}.  %:= [0,s] \times [0,d]^{\nor-1}$. 
The region $\favn_{\text{con}}$ does not constrain $\si{1}$, since it requires that $\si{1}$ not exceed $s$, which follows from the fact that $\svec \in \ssimpn$. Neither does $\favn_{\text{con}}$ constrain $\si{n+1}$.   We will now prove an intermediate result from which $\pcon$ follows; we extend this result to non-uniform i.i.d. parent pdfs in \autoref{Sec:NonUnif}.

We define the simplex  $\mathcal{T}'_{simp}$ and the generic hypercuboid $\mathcal{H}_{gen}$ as: 
\begin{eqnarray}
 \mathcal{T}'_{simp} &:=& \{ \svec_{1:n} \in \R_+^n : \1^T \svec_{1:n} \leq b  \},  \\
 \mathcal{H}_{gen} &:=&  \prod_{i=1}^{\nor}~ [0,a_i]. 
\end{eqnarray}
Let $\favn'_{gen} := \mathcal{T}'_{simp} \cap \mathcal{H}_{gen}$ and $\mathbf{a} = \begin{bmatrix} a_1 & \ldots & a_n\end{bmatrix}^T$.

\vspace{2mm}

\begin{lemma}
\label{thm:vhsp2}
 The volume of $\favn_{gen}'$ is given by
\begin{align}
\label{eqn:vhsp2}
\Vol(\favn_{gen}') =  \frac{1}{n!} \sum _{ \vvec \in \{0,1\}^{n}} (-1)^{\1^T \vvec} \left(\max( b-\mathbf{a}^T \vvec, 0)\right)^{\nor}. 
\end{align}
\end{lemma}
\begin{proof}
Define $\mathcal{T}_{simp}$ and $\mathcal{H}_{cube}$   as in \autoref{Subsec:genshi}. Transform  the coordinates $\svec_{1:\nor}$ to $\svec'_{1:\nor}$, where  $\si{i}' := a_i s_i$, and note that the transformed simplex and hypercube are $\mathcal{T}'_{simp}$ and $\mathcal{H}_{gen}$, respectively. Defining $\mathbb{J}_{n\times n}$ as the Jacobian matrix of the transformation, we then have 
\begin{align}
 \Vol(\favn_{gen}) = \det(\mathbb{J}) \cdot \Vol(\favn'_{gen}).
\end{align}
The diagonal entries of $\mathbb{J}$ are $\mathbb{J}_{i,i} = \frac{1}{a_i}$, and the off-diagonal entries of $\mathbb{J}$ are zero. It follows that $\det(\mathbb{J}) = \prod \frac{1}{a_i}$. \autoref{eqn:vhsp2} follows immediately from \autoref{eqn:vhsp}.
\end{proof}

The slack simplex $\ssimpn_{\rm full}$ and the hypercuboid $\hyp{\text{con}}$ correspond to $\mathcal{T}'_{simp}$ and $\mathcal{H}_{gen}$, respectively.  Define $\mathbf{a} \in \mathbb{R}^n$ as a vector with $a_1 = s$  and $a_{2:n} = d$.  From \autoref{eqn:vhsp2}, we have: % of dimensions of $\hyp{\text{con}}$, 
\begin{align}
\label{eqn:vfavcon1}
\Vol(\favn_{\text{con}}) = \frac{1}{n!} \sum _{ \vvec \in \hyp{\text{con}}} (-1)^{\1^T \vvec} (\max( s-\mathbf{a}^T \vvec, 0))^{n}.
\end{align}

Observing that  when $v_1=1$, we have that 
\begin{align}
\avec^T \vvec = s + {\sum_{i=2}^n d v_i} \geq s \implies  s - \avec^T \vvec ~{\leq}~ 0, \nonumber
\end{align}
so that the resulting simplex contributes no volume to $\favn_{\text{con}}$. Now suppose that $v_1=0$ and that $k$ bits among $\vvec_{2:n}$ are ones. The resulting simplex contributes a volume proportional to $(s-kd)^n$, and there are $\binom{n-1}{k}$ such simplices. Summing the volumes of these simplices  using the PIE, we get $\Vol(\favn_{\rm con})$, from which we obtain
\begin{align}
\label{eqn:pcon}
 \pcon := \frac{\Vol(\favn_{\rm con})}{\Vol(\ssimpnf)} = \sum_{i=0}^{\nmin } (-1)^{i} \binom{n-1}{i} \left(1-\frac{id}{s}\right)^{n},
\end{align}

which agrees with the result in \cite{godehardt1996connectivity}.

\subsection{Number of Connected Components, $\cmp$}
\label{Subsec:cmp}
We will first determine the probability mass function (pmf) $\Prob(\cmp =k)$ of the discrete rv $\cmp$.  First, note that $\Prob(\cmp=1)$ is equal to $\pcon$. If $\cmp = k+1$, then $\svec_{2:n}$ has exactly $k$ disconnected slacks. Let $\favn(\cmp=k+1)$ denote the subset of $\ssimpnf$ consisting of these slacks, and let $\unfavn(\vvec)$ denote the subset of $\ssimpnf$ consisting of slack vectors that obey $\svec_{1:n} \geq  d\vvec$. The subset $\favn(\cmp=k)$ has  $k-1$ disconnected slacks, not including the first unconstrained slack $s_1$. Let $\mathbf{V}(k)$ denote the set of bit vectors encoding the  indices of disconnected slacks in slack vectors with $k$ disconnected slacks (or equivalently, slack vectors having $k+1$ components):
\begin{align}
\label{eqn:vkm1}
\mathbf{V}(k) =  \{ \vvec \in \{0,1\}^n: v_0 = 0 ~\text{and}~ \1^T\vvec = k \}, 
\end{align}
so that 
\begin{align}
\label{eqn:cmpk}
\favn(\cmp=k) = \bigcup_{\vvec \in \mathbf{V}(k-1)} \unfavn(\vvec).
\end{align}

Following the same reasoning as in \autoref{Subsec:pmon}, the quantity $\sum \unfavn(\vvec)$ overestimates $\Vol(\favn(\cmp=k))$; applying the PIE gives us 
\begin{align}
\label{eqn:favcmpk}
 \Vol(\favn(\cmp=k)) &= \frac{1}{n!} \sum _{\vvec \in \mathbf{V}(k-1,\nmin-1)} (-1)^{\1^T \vvec} ( s-d\mathbf{1}^T \vvec)^{\nor},
\end{align}
where 
\begin{align}
\label{eqn:vkm1plus}
\mathbf{V}(k-1,\nmin-1):= \bigcup_{i=k-1}^{\nmin-1} \mathbf{V}(i).  % \bigcup_{i:k-1\leq i} \mathbf{V}(i) 
\end{align}

We can then derive $P(\cmp=k) = \frac{\Vol(\favn(\cmp=k))}{\Vol(\ssimpnf)}$, which simplifies to the expression
\begin{align}
 \Prob(\cmp =k) = \sum_{j=k-1} ^{\nmin-1} (-1)^{j+k-1}   \binom{n-1}{j} \binom{j}{k-1} \left(1-\frac{jd}{s}\right)^{n},
\end{align}

which agrees with the corresponding formula in \cite{godehardt1996connectivity}. We omit the full derivation for conciseness.

\todo{Proof: quote  Feller Volume 1 , page 106: Combinations of Events}

Now we will compute $\Exp(\cmp)$. If a particular slack  $s_j$ is disconnected, irrespective of the other slacks, then the resulting slack vector obeys the constraint $\1^T \svec_{j\neq i} \leq (s-d)$. From \autoref{Subsec:pmon}, this smaller simplex has a volume proportional to $(s-d)^{n}$, so that
\begin{align}
\label{eqn:s1con}
 \Prob(S_i > \con) = \left(1-\frac{d}{s}\right)^{n} ~~\textrm{and}~~~ \Prob(S_i \leq \con) = 1- \left(1-\frac{d}{s}\right)^{n}.
\end{align}
Then, by \autoref{eq:expcmp},
\begin{align}
\label{eqn:expcmp}
 \Exp(\cmp) = 1 + (n-1)  \left(1-\frac{d}{s}\right)^{n}.
\end{align}

\subsection{Sensed Length, $\slen$}
\label{Subsec:slen}
We define $\favn_{\sen}$ as the subset of $\ssimpn$ for which the boundary is fully sensed. Analogous to $\favn_{\mon}$, we express $\favn_{\sen}$ as the intersection between $\ssimpn$ and the sensing hypercuboid $\hyp{\sen} = [0,d]\times[0,2d]^{n+1}\times[0,d]$. The volume of $\favn_{\sen}$ is computed from Lemma \autoref{thm:vhsp2}.  The value of $\Exp(\slen)$ can then be determined from \autoref{eqn:expslen}. 

\vspace{2mm}

\begin{theorem}
\label{thm:pslen}
The probability $\psen$ that $\Bo$ is fully sensed is:
 \begin{align}
 \label{eqn:psen}
 \psen = 1 &- \sum_{i=1} ^{\nmin} (-1)^{i-1} \bigg[\binom{n-1}{i} \max\left(1-2i \frac{d}{s},0\right)^{n}  \nonumber \\
						 &+ ~2  \binom{n-1}{i-1} \max\left(1-  (2i -1)\frac{d}{s},0\right)^{n} \nonumber \\
						&+ \binom{n-1}{i-2} \max\left(1-(2i+1)\frac{d}{s},0\right)^n \bigg] .  
 \end{align}
 Here we adopt the convention that $\binom{n-1}{i-2}=0$ for $i=1$.
\end{theorem}
\begin{proof}
 We derive \autoref{eqn:psen} by analogy with $\pmon$. When computing $\pmon$, every bit vector with $i$ ones led to a choice of $\binom{n+1}{i}$ simplices in \autoref{eqn:unfavnmon}, each contributing a probability of $\max(1-id/s,0)^{n}$ in \autoref{eqn:pmon}. Here, we need to treat the slacks $s_1$ and $s_{n+1}$ differently from  the others. Consider the set of all $n$-bit vectors $\vvec_{1:n}$ that iterate through the vertices of $\hyp{\sen}$. 
 \begin{enumerate}
  \item If $v_1=v_{n+1}=0$, then $i$ of the remaining $n-1$ bits must be equal to $1$. The set of all such $\vvec$ creates $\binom{n-1}{i}$  simplices  of the form $\vvec ^T\svec = s-2id$ that contribute to $\Vol(\favn_{\sen})$.
  \item If $v_1=0, ~v_{n+1}=1$ or $v_1=1, ~v_{n+1}=0$, then $i-1$ of the remaining bits must be equal to $1$. The set of all such $\vvec$ creates $2\binom{n-1}{i-1}$ simplices of the form $\vvec^T \svec= s - (2i-1)d$ that contribute to $\Vol(\favn_{\sen})$ . 
  \item If $v_1=v_{n+1}=1$, then $i-2$ of the remaining $n-1$ bits are equal to $1$. Consequently, the set of all such $\vvec$ creates $\binom{n-1}{i-2}$ simplices. Each simplex has the form $\vvec^T \svec = s-(2i+1)d$ and contributes to $\Vol(\favn_{\sen})$.
 \end{enumerate}
\end{proof}

\subsection{Vertex Degree, $\dgr$}
\label{Subsec:vertdeg}
When two points $(x_1,x_2)$ are selected at random on $[0,s]$, we may write their joint pdf in terms of their order statistics:
\begin{align}
\label{eqn:fjoint2}
 f(\pos{1},\pos{2}) = \frac{2}{s^2} \ind_{0\leq \pos{1}\leq \pos{2} \leq s}. 
\end{align}

Using \autoref{eqn:fjoint2}, the probability term in \autoref{eq:expdgr1} is 
\begin{eqnarray}
&\Prob&\hspace{-3mm}(|X_1 - X_2| \leq \con)  \nonumber \\
&=& \frac{2}{s^2} \left( \int_{\pos{i} =0} ^{s-d} \int_{\pos{2}=\pos{1}} ^{\pos{1}+d} {\rm d} \pos{2} {\rm d} \pos{1} + \int_{\pos{1} =s-d} ^{s} \int_{\pos{2}=\pos{1}} ^{s} {\rm d}\pos{2} {\rm d} \pos{1} \right) \nonumber \\
 &=& \frac{2d s-d^2}{s^2},
\end{eqnarray}
which yields
\begin{align}
\label{eqn:expdeg}
 \Exp(\dgr) = (n-1)\frac{2d s-d^2}{s^2}.
\end{align}

\section{CF Coverage with I.I.D. Uniform Parents}
\label{Sec:UnifCF}

\end{comment}

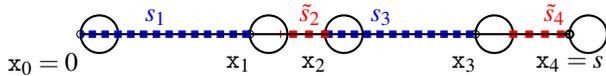
\begin{figure}
   \begin{tikzpicture}

\coordinate(p0) at (-2.25,0);
\coordinate(p4) at (4.25,0);

\coordinate  (a) at (0,0);
\coordinate  (b) at (1,0);
\coordinate  (c) at (3.,0);

\draw[thick] (a) circle (0.25cm);
\node[draw=none,align=left] at (-0.4,-0.4) { $\pos{1}$} ;
\draw[black] (-0.25,0) circle (.05cm); 

\draw[thick] (b) circle (0.25cm);
\node[draw=none,align=left] at (0.6,-0.4) { $\pos{2}$} ;
\draw[black] (0.75,0) circle (.05cm);

\draw[thick] (c) circle (0.25cm);
\node[draw=none,align=left] at (2.6,-0.4) { $\pos{3}$} ;
\draw[black] (2.75,0) circle (.05cm); 

\draw[thick] (p0) circle (0.25cm);
\node[draw=none,align=left] at (-3.0,-0.4) { $\pos{0}=0$} ;
\draw[black] (-2.5,0) circle (.05cm); 

\draw[thick] (p4) circle (0.25cm);
\node[draw=none,align=left] at (4.0,-0.4) { $\pos{4}=s$} ;
%\draw[black] (5.25,0) circle (.05cm); 

     \draw[color=blue] (-1.5,0) node[anchor=south] {$\si{1}$};
     \draw[color=blue][dotted][line width=.1cm] (-.25,0) -- (-2.5,0);

     \draw[color=blue] (1.5,0) node[anchor=south] {$\si{3}$};
     \draw[color=blue][dotted][line width=.1cm] (0.75,0) -- (2.75,0);

      \draw[color=red] (0.55,0) node[anchor=south] {$\fsi{2}$};
     \draw[color=red][dashed][line width=.1cm] (0.8,0) -- (0.15,0);

         \draw[color=red] (3.8,0) node[anchor=south] {$\fsi{4}$};
     \draw[color=red][dashed][line width=.1cm] (3.2,0) -- (4,0);

\draw[thick] (-2.5,0) -- (+4,0);

%\draw[black,fill=white] (-2,0) circle (.1cm); 
\draw[thick] (4,0) circle (.05cm);

%s\draw[dashed,thick,green] (-2,0.24) -- (5,0.24);

\end{tikzpicture}
\caption{Slacks (blue) and free slacks (red) of the robot configuration in \autoref{fig:Attach2}.  }
\label{fig:Att3}
\end{figure}

%\end{document}

We will now analyze the graph $\G$ of a CF configuration and the resulting CF geometric objects.   In \autoref{fig:Att3}, the left virtual robot $\Rob_{0}$ has position $\upos{0}=\pos{0} = 0$ and occupies the interval $[0,\dia]$; the right virtual robot is located at $\pos{n+1} = s$, occupying the interval $[s,s+D]$. Robots $\Rob_{1:n}$ are located at positions on the interval $[\dia,s-\dia]$, which has length $s-2\dia$.   

We define $\ssimpcf$ to be the subset of $\ssimpn$ consisting of CF slack vectors. An alternative way to characterize $\ssimpcf$, to be used later, is through \textit{free slacks}. We define the \textit{free slack} $\fsi{i}$ associated with $\si{i}$ to be the length of that subinterval of $[\pos{i-1},\pos{i}]$ that falls outside the body of any robot. From \autoref{fig:Att3}, we see that $\fsi{i} = \si{i} - \dia$. We define the vector of free slacks as $\fsvec_{1:n+1}$, which must satisfy
\begin{align}
\label{eqn:fsvec}
 \1^T \fsvec_{1:n+1} = \fsla := s - (n+1) \dia.
\end{align}
Here, $\fsla$ is the total free slack, which is the effective amount of slack in a CF robot configuration. We call the set of slacks obeying \autoref{eqn:fsvec} the \textit{free slack simplex} $\tilde{\ssimpn}$. Analogously to $\ssimpnf$, we define the full-dimensional equivalent of $\ssimpn_{\rm CF, full}$  by
\begin{align}
 \ssimpn_{\rm CF,full} := \{\fsvec_{1:n} \in \R_+^n : \1^T\fsvec \leq \fsla \}.
 \end{align}
 
The volume of $\ssimpcf$ remains the same regardless of whether it is expressed in terms of free slacks or slacks, so that
\begin{align}
\label{eqn:ssimpcf}
 \Vol(\ssimpn_{\rm CF,full}) =  \frac{\fsla^n}{n!}.
\end{align}

We will assume that $\ssimpn$ and $\favn$ are represented in full-dimensional form throughout this section. For simplicity, we will not discuss the constraint $\si{n+1}\in [\dia,\con]$ on the last slack, which can be handled in a fashion similar to equations \ref{eqn:2inter} and \ref{eqn:2intervol}. 

\subsection{Probabilities $\pmon$ and $\pcon$}
\label{Subsec:pmoncf}
We will first introduce the geometric concepts for computing $\pmon$. Analogous to the CT case, we define the monitoring hypercube to be $\hyp{\rm CF, \mon} =  [\dia,\con]^{n+1}$ and the favorable region for monitoring to be $\favn_{\rm CF,\mon}:= \ssimpnf \cap \hyp{\rm CF, \mon}$. Note that the CF constraint is handled in the hypercube, and not in the slack simplex. Since $\hyp{\rm CF, \mon}$ is not of the form $[0, c]^{n+1}$ for some constant $c$, the regions defined by $ \ssimp(\vvec):= \{\svec \in \ssimpn: \svec \geq d \vvec \} \cap \hyp{\rm CF, \mon}$ are no longer simplices, and so we cannot apply Lemma \autoref{thm:vhsp} directly to compute $\Vol(\favn_{\rm CF,\mon})$. This problem cannot be remedied by transforming slacks into free slacks.  Instead, we extend  Lemma \autoref{thm:vhsp} to the case where the simplex $\mathcal{T}_{simp}$, defined in \autoref{eqn:tsimp}, is intersected with a {\it displaced} hypercuboid.  We define a displaced hypercuboid by specifying its diagonally opposite vertices, $\vvec_c$   and $\vvec_f$, which are respectively its closest vertex to $\0$ and its farthest vertex away from $\0$: 
\begin{align}
\label{eqn:dispcub}
\hyp{\vvec_c,\vvec_f} = \{ \svec_{1:n} \in \R_+^n : \0 \leq \vvec_c \leq \svec \leq \vvec_f \}. 
\end{align}
The vertex pair $(\vvec_c,\vvec_f)$ encodes all information about the faces of the hypercuboid. Now we will use \autoref{alg:geninter} to compute the volume of the intersection $\favn:= \mathcal{T}_{simp} \cap \hyp{\vvec_c,\vvec_f}$. To do so, the algorithm computes the volume of $\mathcal{T}_{simp} \cap \hyp{\0, \vvec_f}$, which overestimates $\Vol(\favn)$, in step \ref{step:overV}. Subsequently, it subtracts the volumes of overlap between $\mathcal{T}_{simp}$ and the exterior of $\hyp{\vvec_c,\vvec_f}$, i.e. the region $\hyp{\0, \vvec_f} \setminus \hyp{\vvec_c, \vvec_f}$, using the PIE. This exterior is the union of hypercuboids $\hyp{b,i}$ that lie between $\hyp{\vvec_c,\vvec_f}$ and $\0$. The volume of intersection of $\mathcal{T}_{simp}$ with this exterior, denoted by $\hyp{inter}$, is computed using the PIE in step \ref{step:compvol}, and either added to or deducted from $V$ in step \ref{step:extpie}. \autoref{Fig:CFInter} illustrates the operation of \autoref{alg:geninter} in two dimensions.

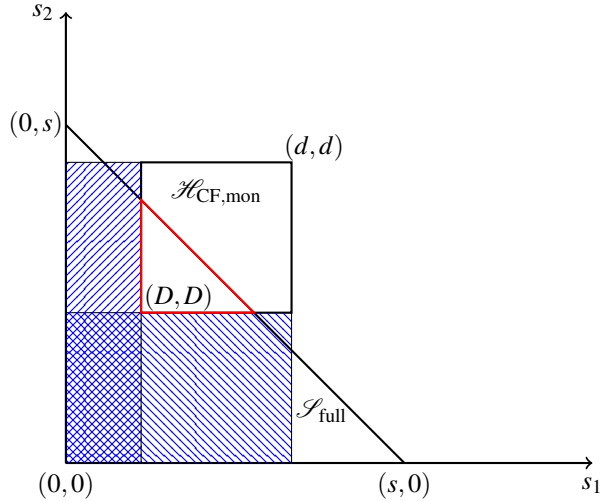
\begin{figure}[t]
\usetikzlibrary{arrows,snakes,backgrounds,patterns,matrix,shapes,fit,calc,shadows,plotmarks}
\begin{tikzpicture}

\draw[thick,->] (0,0)  -- (7,0);
\draw[thick,->] (0,0) -- (0,6);

% Hypercube 
\draw[thick] (1,2) -- (3,2) -- (3,4) -- (1,4) -- (1,2);

% Simplex
\draw[thick] (4.5,0) -- (0,4.5);

% CF 1
\draw[pattern=north west lines, pattern color=blue] (0,0) rectangle (3,2);
\draw[pattern=north east lines, pattern color=blue] (0,0) rectangle (1,4);

\draw(0,-0.3) node {$(0,0)$};
\draw(7,-0.3) node {$\si{1}$};
\draw(-0.3,6) node {$\si{2}$};

\draw(4.5,-0.3) node {$(s,0)$};
\draw(-0.4,4.5) node {$(0,s)$};

\draw(1.5,2.2) node {$(D,D)$};

\draw(3.3,4.2) node {$(d,d)$};

\draw(2.0,3.6) node {$\hyp{\rm CF,mon}$};
\draw(3.4,0.7) node {$\ssimpnf$};

  \draw[red,thick] (1,2) -- (2.5,2) -- (1,3.5) -- cycle;

\end{tikzpicture}
\caption{Illustration of \autoref{alg:geninter} for $n=2$. The simplex $\ssimpnf$ and the hypercube $\hyp{\rm CF,\mon}$ are represented in full-dimensional form, and constraints on $\si{3}$ are not shown for simplicity.  $\ssimpnf$ is the triangle with vertices $\{(0,0),(0,s),(s,0)\}$ and $\hyp{\rm CF,mon}$ is the square $[D,d]^2$. \autoref{alg:geninter} computes the area of intersection of $\ssimpnf$ with $[0,d]^2$, from which it subtracts the  area of $\ssimpnf$ under the hatched region to obtain the area of the red triangle $\favn_{\rm CF,mon}$.}
\label{Fig:CFInter} 
\end{figure}

\vspace{5mm}

\begin{algorithm}
 \caption{Find volume of intersection between a general simplex and a displaced hypercuboid}
 \begin{algorithmic}[1]
 \Procedure{FIND-INTER-VOL}{$\mathcal{T}_{simp}(\avec,b),\mathcal{H}_{\vvec_c,\vvec_f}$}
 \State $\mathbf{d}_{1:n} \gets \mathbf{v}_f - \mathbf{v}_c$  \Comment{Edges of $\hyp{\vvec_c,\vvec_f}$}
 \State $V \gets \Vol(\mathcal{T}_{simp} \cap \hyp{\0, \vvec_f})$ \Comment{Compute using Lemma \autoref{thm:vhsp2}} \label{step:overV}
 \For{$i\gets 1\ldots n$} %\Comment {Define $\nor$ hypercuboids ``below'' $\hyp{gen}$ } 
  \State $\hyp{b,i} = \hyp{\0, \vvec_i}$ where $\vvec_i = \vvec_c + (\vvec_{f,i} - \vvec_{c,i}) \mathbf{e}_i$ %as hypercuboid with opposite vertices at $\0$ and $\vvec_c + (\vvec_{f,i} - \vvec_{c,i}) \mathbf{e}_i$. 
 \EndFor
 \For{$\mathbf{v} \in \{0,1 \}^{n}$}  \label{step:viter}
  \State $\hyp{inter}(\vvec) \gets \bigcap_{v_i =1} \hyp{b,i} $
  \State Compute $\Vol(\mathcal{T}_{simp} \cap \hyp{inter})$ using Lemma \autoref{thm:vhsp2} \label{step:compvol}
  \State $V \gets V + (-1)^{\1^T \vvec} \Vol(\hyp{inter}) $ \label{step:extpie} %\Comment {$\Vol(\hyp{inter})$ from Lemma \autoref{thm:vhsp2}} 
 \EndFor
 \State Return $V$
 \EndProcedure
 \end{algorithmic}
 \label{alg:geninter}
\end{algorithm}

The hypercube $\hyp{\rm CF, \mon}$ is an instance of $\hyp{\vvec_c,\vvec_f}$ with $\vvec_c = \dia \1^T$ and $\vvec_f = \con \1^T$. Since $\hyp{\rm CF,\mon} = [\dia,\con]^n$ is a hypercube, it has identical faces, so  each of the hypercuboids $\hyp{b,i}$ has the same set of edges. Further, the simplex $\ssimpn$ is an instance of $\mathcal{T}_{simp}$ with $\avec=\1$; since the coefficients $a_i$ are all unity, it follows that $\hyp{inter}(\vvec)$ and $\hyp{inter}(\vvec')$ are congruent for any two $n$-bit vectors $\vvec$ and $\vvec'$ with the same number of 1-bits.  Thus, we need to compute the volumes of regions of the form $\hyp{inter} (\vvec)$, where $\vvec$ consists of a set of successive ones followed by zeros; for example, $\vvec = (1,1,...,1,0,...,0)$. In other words, $\vvec$ can be expressed as the sum of unit vectors $\sum_{j=1}^{i}\mathbf{e}_j$. We thus have 
\begin{align}
\label{eqn:volfavcf}
\Vol&(\favn_{CF,\mon}) = \Vol(\ssimp \cap [0,\con]^{n+1}) - \nonumber \\ 
&\sum_{i=1}^{n+1} (-1)^{i} \binom{n+1}{i} \Vol \left(\ssimp \cap \hyp{inter}(\sum_{j=1}^i \mathbf{e}_j) \right),
\end{align}
and $\pmon = \frac{\Vol(\favn_{CF,\mon})} {\Vol(\ssimpn)}$. 

Running \autoref{alg:geninter} with inputs $\mathcal{T}_{simp}:=\ssimpnf$ and $\hyp{\vvec_c, \vvec_f}:=\hyp{\rm CF,con} = [\dia,s] \times [\dia,\con]^{n-1}$ furnishes the volume of the favorable region $\favn_{\rm CF,\text{con}}$ for a connected CF configuration, and thereby $\pcon$. %\todo{Briefly explain how}

\subsection{Probability  $\psen$}
Computing $\psen$ requires us to consider the intersection between $\ssimpn$ and the hypercuboid
\begin{align}
 \hyp{\rm CF, \sen}:= [D,d] \times [D,2d]^{n-1} \times[D,d].
\end{align}
Both $\ssimpn$ and $\hyp{\rm CF,\sen}$  are degenerate polytopes. To compute their volume of intersection, we express the intersection $\favn_{\rm CF, \sen}$ in full-dimensional form, as in \autoref{eqn:2inter}. Define $\hyp{\rm full}:= [D,d]\times [D,2d]^{n-1}$ to be the full-dimensional form of $\hyp{\rm CF, \sen}$. Following the approach of \autoref{eqn:2inter},  
we can compute the full-dimensional form of $\favn_{\rm CF,\sen}$ as
\begin{align}
 \favn_ {\rm full, sen} : = \hyp{\rm full} \cap \ssimpnf \cap \ssimpnf(s_{n+1} \in  [D,d]), 
\end{align}
where $\ssimpnf(s_{n+1} \in  [D,d])$ is defined as the subset of $\ssimpnf$ in which the last slack $\si{n+1}=s-\1^T\svec_{1:n}$ lies in $[D,d]$.  Then we can determine the volume of $\favn_{\rm full,\sen}$ as:
\begin{align}
\label{eqn:2intersen}
\Vol(\favn_{\rm full,\sen}) &= \Vol(\ssimpnf \cap \hyp{\rm full}) \nonumber \\
  & ~~~~ - \Vol({\ssimpnf}(s_{n+1} > d) \cap \hyp{\rm full}) \nonumber \\
  & ~~~~ - \Vol({\ssimpnf}(s_{n+1} < D) \cap \hyp{\rm full}),
\end{align}
where the region ${\ssimpnf}(s_{n+1} > d)$ is the subset of $\ssimpnf$  over  which $\si{n+1}$  is disconnected, and likewise ${\ssimpnf}(s_{n+1} < D)$ is the subset of $\ssimpnf$ for which the last slack results in a conflict. We run \autoref{alg:geninter} once for each volume of intersection in \autoref{eqn:2intersen} to obtain $\Vol(\favn_{\rm full,\sen})$. We can then compute $\psen=\frac{\Vol(  \favn_{\rm CF, \sen})} {\Vol(\ssimpnf)}$ and, by \autoref{eqn:expslen}, $\Exp(\slen) = s \cdot \psen$.

\subsection{Number of Connected Components, $\cmp$}
\label{Subsec:cmpcf}
Proceeding as in \autoref{Subsec:cmp}, we define $\favn_{\text{CF}}(\cmp=k)$ to be the subset of $\ssimpcf$ whose slacks have $k$ connected components. We now write the CF equivalent of \autoref{eqn:favcmpk}, replacing $\favn(\cmp=k)$ with its CF counterpart $\favn_{\text{CF}} (\cmp=k)$. To compute  the volume of this region, we can use \autoref{alg:geninter} with the inputs
\begin{align}
 \mathcal{T}_{simp} := \ssimpnf ~\text{and}~ \hyp{\vvec_c,\vvec_f} := [0,s] \times [d-D]^{n-1} \nonumber,
\end{align}
and execute the loop in step \ref{step:viter} of \autoref{alg:geninter} only over bit vectors obeying \autoref{eqn:vkm1plus}. 

\subsection{Free Slack Approximation of Coverage Properties}
\label{Subsubsec:fslapp} 
The probability $\pmon$ is more complex to compute in the CF case, where  \autoref{alg:geninter} must be used, than in the CT case, where \autoref{eqn:pmon} can be applied. To aid the design and implementation of algorithms for CF coverage, we approximate $\favn_{\text{CF},\mon}$ by $\tilde{\mathcal{F}}_{\text{CF},\mon}$, defined as the intersection between the free slack simplex $\tilde{\ssimpn}$  and the hypercube $\tilde{\mathcal{H}}_{\mon}:=[0,\tilde{d}]^{n+1}$, where $\tilde{d}:=d-D$. Likewise, by intersecting $\tilde{\ssimpn}$ with $\tilde{\mathcal{H}}_{\text{con}}:=[0,\fsla] \times [0,\tilde{d}]^{n-1} \times [0,\fsla]$, we obtain an approximation of the favorable region of connectivity, $\tilde{\mathcal{F}}_{\text{CF,con}}$. The volumes of both these intersections have the form of \autoref{eqn:vhsp}. Hence, the formulas for the resulting probabilities of monitoring and connectivity, denoted by $\tilde{p}_{\mon}$ and $\tilde{p}_{con}$, can be expressed as the equations \ref{eqn:pmon} and \ref{eqn:pcon} for $\pmon$ and $\pcon$ with $(\fsla,\tilde{d})$ substituted for $(s,d)$.  

Similarly, we can use these parameter substitutions in equations \ref{eqn:expcmp}, \ref{eqn:psen}, and \ref{eqn:expdeg} to approximate  $\Exp(\cmp)$, $\psen$ (and thus $\Exp(\slen)$), and $\Exp(\dgr)$, respectively.  We call this approximation the \textit{free slack approximation} (FSA) for the CF case.  
\section{Approximations of Coverage Properties}
\label{Sec:Estim}
 In this section, we develop order-of-magnitude approximations of the graph properties derived in \autoref{Sec:Unif} and \autoref{Sec:UnifCF}, using the concepts of {\it threshold functions} \cite{friezeRandomGraphs, goel2004sharp, muthukrishnan2005bin} and \textit{Poissonization} 
\cite{barbour1992poisson} that are described in \autoref{Sec:Bground}. We will use threshold functions in our design procedure in \autoref{Sec:Design}.  We note that the estimates below are also valid for CF configurations when the FSA substitutions $(s,d) \rightarrow (\fsla,\tilde{d})$ are applied. 

\subsection{Connectivity and Monitoring Thresholds}
\label{Subsec:connthresh}
The uniform parent pdf has the special property that the slack vector $\rvsvec$ is jointly uniform over $\ssimpn$, with each slack being identically distributed (though not independent) as  scaled exponentials of the form $s \cdot \mathrm{Exp}(1)$. Suppose we sort the slacks in $\rvsvec$ in increasing order to obtain the vector of their order statistics, $\orsvec$. Then we find that  \cite{ostat2003} 
\begin{eqnarray}
 \Exp({\orvsi{i}}) &=& \frac{s}{n+1} \sum \limits_{j=i} ^{n+1} \frac{1}{j} ~=~ \frac{s}{n+1} (H_{n+1} - H_{i}), \\
 \Var({\orvsi{i}}) &=& \sum \limits_{j=i} ^{n+1} \frac{1}{j^2},
\end{eqnarray}
where $H_n$ denotes the harmonic numbers. If $\nor$ is large,  we may approximate  $H_n$ by  $\log n$.  The longest slack $\orvsi{n+1}$ has the expected value $\Exp(\orvsi{n+1}) = \frac{s}{n+1}H_{n+1} \approx \frac{s}{n+1}\log(n+1)$ for large $n$.  To enforce the monitoring condition, we impose the constraint that the expected longest slack $\Exp(\orvsi{n+1})$ is connected,
%, yielding $\Exp({\orvsi{n+1}}) \approx \frac{s}{n+1}\log(n+1)$.
\begin{align}
\label{eqn:harmsd}
 \frac{\log(n+1)}{n+1} \leq  \frac{d}{s} ~\implies~ n = \exp\left(-\mathrm{W}\left(\frac{d}{s}\right)\right) - 1,
\end{align}
where $W$ is the Lambert W function \cite{yi2010time}. 
 
From \cite{Matthew2003RGG}, a pair $(n,d)$ with $d =  O(\frac{n} {\log n})$ forms a \textit{connectivity threshold} for $\G$, implying that the estimate for $n$ is sharp. Further, since this function thresholds connectivity, it automatically thresholds monitoring too.  We can use a result from \cite{muthukrishnan2005bin} to determine that a triple $(n,d,s)$ satisfying  
\begin{align}
\label{eqn:shthresh}
nd = \Theta(s\log s ).
\end{align}
forms a sharp {\it monitoring threshold}.

\subsection{Threshold for $\Exp(\dgr)$}
\label{Subsec:expdeg}
From \cite{Matthew2003RGG}, a sequence of pairs $(n,d)$ satisfying  $d = O(\frac{1}{n})$ results in $\Exp(\dgr)$ tending to a positive constant, and this sequence is  called the \textit{thermodynamic limit}. Likewise, pairs  $(n,d)$ that obey $d= O(\frac{\log n}{n})$ cause $\Exp(\dgr)$ to have growth of order $O(\log n)$ and comprise the \textit{connectivity regime}, which is a threshold for the property that $\G$ has no isolated vertices (i.e. vertices of zero degree). A {\it superconnectivity regime} has $d = \Omega(\frac{\log n}{n})$, which ensures that $\G$ has no isolated vertices a.s. Likewise, the {\it subconnectivity  regime} with $d=o(\frac{\log n}{n})$ ensures that $\G$ has one or more isolated vertices a.s.

\subsection{Poisson Approximation of Coverage Properties}
\label{Subsec:paunif}

\subsubsection{Poissonizing $\Nmon$ and $\Ncon$} 
\label{Subsec:ppmon}
We derive Poisson estimates for $\Nmon$ and $\Ncon$ using results from  \cite[Ch.7]{barbour1992poisson}. The key observation is that the slack rv's  $\rvsla{i:1,\ldots,n+1}$, are negatively associated, since an increase in one slack leads to a corresponding decrease in the others.  Applying Lemma \autoref{thm:poisson} to $\Nmon$, we see that $\Nmon$ is approximately distributed as $\Poi (\lambda_{\mon})$ with $\lambda_\mon = \sum_{i=1}^{n+1}\Exp( \ind_{S_i \leq \con})$.  The slacks' exchangeability implies that $S_1\sim S_i$, so that $\lambda_\mon = (n+1) \Prob (S_1\leq \con)$.
From  \autoref{eqn:s1con}, we obtain
\begin{align}
\label{eqn:lambdamon}
 \lambda_\mon = (n+1)  \left(1- \left(1-\frac{d}{s}\right)^{n}\right).
\end{align}
Using \autoref{eq:Ncon}, the analogous Poisson rv for $\Ncon$ has mean $\lambda_{\mathrm{con}}=(n-1) \Prob(S_1\leq \con)$.  These Poisson approximations are valid for regimes $(n,d)$ for which \cite[Ch.7]{barbour1992poisson}
\begin{align}
\label{eqn:limpoi}
 n \rightarrow \infty, ~\frac{d}{s}\rightarrow 0, ~\text{and}~ n \frac{d}{s} \rightarrow \frac{\lambda_\infty + o(1)}{n}, 
\end{align}
where $\lambda_\infty$ is a finite constant. 

We may approximate $\pmon$ using $\Nmon$ by noting that when $\Nmon=n+1$, the entirety of $\Bo$ is monitored. Let $\Lambda \sim \Poi(\lambda_{\mon})$; then we have that
\begin{align}
\label{eqn:poipmon}
\pmon := \Prob(\Nmon=n+1) \approx \Prob (\Lambda = n+1) \nonumber \\ = \frac{ \lambda_{\mon}^{n+1}  \exp(-\lambda_{\mon}) }{(n+1)!}.
\end{align}

%We may approximate $\pmon$ using $\Nmon$ by noting that when $\Nmon=n+1$, the entirety of $\Bo$ is monitored, so that
%\begin{align}
%\label{eqn:poipmon}
% \pmon := \Prob(\Nmon=n+1) \approx \Prob (\Poi(\lambda_{\mon} = n+1)) \nonumber \\ = \frac{ \lambda_{\mon}^{n+1}  \exp(-\lambda_{mon}) }{(n+1)!}.
%\end{align}

Likewise, we have for $\pcon$ :
\begin{align}
 \label{eqn:poipcon}
 \pcon := \Prob(\Ncon=n-1) \approx \frac{ \lambda_{\rm con}^{n-1}  \exp(-\lambda_{\rm con}) }{(n-1)!}.
\end{align}

\subsubsection{Poissonizing $\Nsen$} 
Reasoning similarly for $\pcon$, we have $\Nsen \sim \Poi(\lambda_{sen})$, where 
\begin{align}
 \lambda_{\sen} &= \sum_{i=2}^{n}  \Exp(\ind_{S_i \leq 2\con}) + \Exp(\ind_{S_1 \leq \con}) + \Exp(\ind_{S_{n+1} \leq \con})  \nonumber \\
 &= n  \left(1- \left(1-\frac{2d}{s}\right)^{n}\right) + 2  \left(1- \left(1-\frac{d}{s}\right)^{n}\right),
 \end{align}
 and thereby approximate $\psen$ as 
 \begin{align}
  \psen = \Prob(\Nsen = n+1) \approx \frac{\lambda_\sen^{n+1} \exp(-\lambda_{\sen})} { (n+1)!},
  \end{align}
  whose approximation error tends to zero when \autoref{eqn:limpoi} holds.

\subsubsection{Poissonizing $\cmp$}
\label{Subsec:estimcomp}
Using definition \autoref{eqn:cmpnum} for $\cmp$, we can apply our approach for estimating $\pmon$ and $\pcon$ to obtain $\cmp \sim \Poi(\lambda_{\cmp})$, where  
\begin{align}
 \lambda_{\cmp} = 1 + (n-1) \Exp(\ind_{S_1 > d}) =  1+ (n-1) \left(1- \frac{d}{s}\right)^{n}.
\end{align}

% \cite{barbour1992poisson} Poissonizes the \textit{expected maximal vertex degree} of $\G$, which we will not pursue here. 

\section{Design for Target Coverage Properties}
\label{Sec:Design}

We now apply our analytical results to select the number of robots $n$ that will achieve target properties of $\G$, such as a specified value of $\pcon$, when performing stochastic boundary coverage.  Our design problems require \textit{inverting} the equations that we have derived for these properties.  Although here we assume that the robot diameter $D$, robot communication radius $d$, and boundary length $s$ are fixed, as will be the case in many coverage applications, in general we may also use our procedure to compute these parameters to achieve specified properties of $\G$.  Due to the nonlinear dependency of each property on the parameters $(s,n,\dia,d)$, the inversion is performed using numerical methods. Our source code in  \cite{kumar_polytope_2016} implements this inversion procedure using the $\tt{scipy~ fsolve}$ numerical solver.  

We illustrate the computation of $n$ in both  conflict-tolerant (CT, \autoref{Sec:Unif}) and conflict-free (CF, \autoref{Sec:UnifCF}) coverage scenarios with a uniform i.i.d. parent pdf.  Our solver source code for these two cases can be accessed at \cite{kumarUniformCT2016} and \cite{kumarUniformCF2016}, respectively.  In the CF case, we use the free slack approximation (FSA) to compute $n$ for each property.

In our example, we assume that $s=200$, $d=5$, and $\dia=1$. Our objective is to compute the value of $n$ that yields each of the following properties separately: (1) $\pmon= 0.80$, (2) $\pcon=0.70$, (3) $\Exp(\cmp) = 4$, (4) $\Exp(\slen)=0.6s$, and (5) $\Exp(\deg)=5$.  % While designing $n$ for a particular case, disregard all others.

\begin{enumerate}
 \item $\pmon=0.80$: To compute an initial guess of $n$ for the solver $\tt{fsolve}$, we note that a length of $s_{\mon}=s\pmon = 200(0.80) = 160$ needs to be monitored on average.  Using \autoref{eqn:harmsd}, we solve for an initial guess $n_0$ that satisfies $\log(n_0+1)/(n_0+1) = d/s_{\mon}$, which yields $n_0=162.00$.  In the CT case, we numerically invert \autoref{eqn:pmon} to find $n=  283.15$, and in the CF case, we compute $n= 120.74$.  %, with corresponding $\pmon =0.79$.  

 \item $\pcon=0.70$: We again use the initial guess $n_0=162.00$.  In the CT case, we invert \autoref{eqn:pcon} to obtain $n=261.58$, and in the CF case, we compute $n=116.84$.

 \item $\Exp(\cmp) = 4$: From \autoref{eqn:expcmp}, we observe that for the given values of $d$ and $s$, $\Exp(\cmp)$ has a maximum of $15.53$ at $n_{max}=39.49$. In addition, $\Exp(\cmp) = 0$ at $n=0$ and $\lim_{n\rightarrow \infty} \Exp(\cmp) = 0$. Thus, there are two values of $n$ at which $\Exp(\cmp)=4$: $n_{low} \in (0,n_{max})$ and $n_{high} \in (n_{max},\infty)$.   In the CT case, initializing $\mathtt{fsolve}$ with $n_{0,low} =1$ gave $\nor_{low}=4.34$, and initializing with $n_{0,high} = 162.00$ (as for the previous two properties) gave $\nor_{high}=155.74$. In the CF case, these two initial values $n_{0,low}$ and $n_{0,high}$ produced $n_{low} = 4.27$ and $n_{high} = 90.98$, respectively.

 \item $\Exp(\slen) = 0.6s=120$: We have $\psen= 0.6$. In the CT case, inverting  \autoref{eqn:psen} yields $n=111.77$, and in the CF case, we compute $n=79.08$.  
 
 \item $\Exp(\deg) = 5$: In the CT case,  \autoref{eqn:expdeg} can be directly solved for $n=102.26$.  In the CF case, we compute $n=77.93$.
 
\end{enumerate}

Observe that for each property, the $n$ obtained in the CF case is always less than the $n$ in the CT case. This is as expected, since the robots in the CF case can occupy an effective boundary length of $\fsla < \sla$. For this example, we found that defining $n_0$ from $nd = s \log s$ in \autoref{eqn:shthresh}, instead of from \autoref{eqn:harmsd}, yields the same answers as above for both the CF and CT cases.

\section{Coverage with Non-Uniform I.I.D. Parents}
\label{Sec:NonUnif} 

In general, the case where robot positions on $\Bo$ are distributed according to a {\it non-uniform} parent pdf is more computationally complex than the case of a uniform parent pdf.  For instance,  we proved in \cite{KumarArxiv2016,KumarDissert2016} that computing $\pmon$ for a certain class of parent pdfs is \textit{\#P}-Hard. In this section, we extend the results in \autoref{Sec:Unif} and \autoref{Sec:UnifCF} to the case of nonuniform i.i.d. parents  $ \fpar(x)$.  

\subsection{CT Coverage}

\subsubsection{Spatial pdfs of Robot Positions and Slacks}
\label{Subsec:nonspatial}
The joint and marginal pdfs of the ordered robot positions are given by \autoref{eqn:ostat}. The joint pdf of the slacks is 
\begin{align}
\label{eqn:nonuni-jsla}
 \jsla = \fpar(\si{1}) \fpar(\si{1} + \si{2}) \ldots \fpar(\sum_{i=1}^{n} \si{i}) \ind _{\ssimpn}. 
\end{align}
The  slacks in \autoref{eqn:nonuni-jsla} are  not exchangeable, unlike those induced by the uniform parent. Consequently, \autoref{eqn:nonuni-jsla} is more complicated to integrate than \autoref{eqn:ostat-unisla}. We will find the pdf of slack $S_i:=\rvpos{i+1} -\rvpos{i}$ by noting   that the joint pdf of $\rvpos{i}$ and $\rvpos{i+1}$ is \cite[Ch.2]{ostat2003}
\begin{align}
\label{eqn:jointi}
f_{\rvpos{i},\rvpos{i+1}}&(\pos{i},\pos{i+1}) =  \frac{n!}{(i-1)!(n-i-1)!} \times \nonumber \\ 
&\fpar(\pos{i}) \fpar(\pos{i+1}) G^{i-1} (\pos{i})(1- G (\pos{i+1}))^{n-i-1},
\end{align}
so that 
\begin{align}
\label{eqn:slanu}
f_{S_i} (s_i) = \int_{t=0} ^{s} f_{\rvpos{i},\rvpos{i+1}} (t, t+s_i) dt.
\end{align}
Here, $G$ is the CDF of $g$.

\subsubsection{Probabilities $\pmon$, $\pcon$, and $\psen$}
\label{Subsec:nupmon}
The measure of a subset $\favn_{gen}$ of $\ssimpn$ induced by $\jsla$ is given by  $\Vol(\jsla, \favn_{gen})$.  Using Lemma \autoref{thm:vhsp}, we may express this measure in terms of measures over simplices:
\begin{eqnarray}
\label{eqn:vhspnu}
\hspace{-4mm} \Vol(\jsla, \favn_{gen}) \hspace{-2mm} &=& \hspace{-2mm} \sum _{\vvec \in \{0,1 \}^n } (-1)^{\1^T \vvec} \hspace{1mm} \Vol(\jsla,\Delta(\vvec)).
\end{eqnarray}
This expression requires the computation of $O(2^n)$ integrals.  We can obtain the formulas for $\pmon$, $\pcon$, and $\psen$ by replacing $\Vol(\Delta(\vvec))$ by $\Vol(\jsla,\Delta(\vvec))$ in equations \ref{eqn:pmon}, \ref{eqn:pcon} and \ref{eqn:psen} respectively.  For instance, $\pmon$ and $\pcon$ are given by:  
\begin{align}
 \label{eqn:combint}
 \pmon = 1 -  \sum _{\vvec \in \{0,1\}^n : 1\leq \1^T \vvec \leq \nmin } (-1) ^{\1^T \vvec -1}\Vol(\jsla, \unfavn(\vvec)) ,  
 \end{align}
 \begin{align}
   \pcon =   \sum _{\vvec \in \vm(\hyp{\text{con}}) } (-1) ^{\1^T \vvec } \Vol(\jsla,\Delta(\vvec)).
 \end{align}

We may approximate these probabilities using $\Ncon$, $\Nmon$ and $\Nsen$.  We first derive the probabilities that slack $S_i$ is connected and disconnected:
\begin{align}
\label{eqn:pconnu}
 \Prob(S_i \leq d) &= \int_{ s_i=0}^ {d}  f_{S_i} (s_i) ds_i, \nonumber \\
 \Prob(S_i > d) &= 1- \int_{ s_i=0}^ {d}  f_{S_i} (s_i) ds_i.
\end{align}

Now, proceeding as in \autoref{Subsec:paunif}, we have that $\Nmon$ is approximately distributed as $\Poi(\lambda_{\mon})$, where 
\begin{align}
 \lambda_\mon = \sum_{i=1}^{n+1} \Prob (S_i\leq d),
\end{align}
whose approximation error tends to zero when \autoref{eqn:limpoi} holds. We may then approximate $\pmon$ by \autoref{eqn:poipmon}.  Analogous expressions may be derived for $\pcon$ and $\psen$.

\subsubsection{Number of Connected Components, $\cmp$ }
\label{Subsec:nucmp}
By \autoref{eq:expcmp}, we have that
\begin{align}
 \label{eqn:expcmpnu}
 \Exp(\cmp) = 1 + \sum_{i=2}^{n} \Prob(S_i >d),
\end{align}
which may be determined from \autoref{eqn:pconnu}.

\subsubsection{Vertex Degree, $\dgr$}
\label{Subsec:nudeg}
By \autoref{eq:expdgr1}, we can compute  $\Exp(\dgr)$ from $\Prob(|X_i -X_j| \leq \con)$.The joint pdf of two unordered positions can be rewritten in terms of their order statistics:
\begin{align}
f_{\rvpos{1},\rvpos{2}}(\pos{1},\pos{2}) = 2! \fpar(\pos{1}) \fpar(\pos{2}) \ind_{0\leq \pos{1}\leq \pos{2}\leq s}.
\end{align}
We now obtain
\begin{align}
\label{eqn:nonexpdeg}
\hspace{-3mm} \Prob(|X_2 -X_1| \leq \con) = ~& 2 \int_{\pos{1} = 0} ^{s-d}  \int_{\pos{2} = \pos{1}} ^{\pos{1}+d} \fpar(\pos{1}) \fpar(\pos{2}) d \pos{1} d\pos{2}  \nonumber \\
+ ~& 2 \int_{\pos{1} = s-d} ^{d}  \int_{\pos{2} = \pos{1}} ^{s} \fpar(\pos{1}) \fpar(\pos{2}) d \pos{1} d\pos{2},
\end{align}
which cannot be simplified further. 

\subsection{CF Coverage and Approximations of Coverage Properties}
We can obtain the coverage properties for the CF case by replacing $\Vol(\Delta(\vvec))$ by $\Vol(\jsla,\Delta(\vvec))$ in the formulas in \autoref{Sec:UnifCF}.  In contrast to the results in \autoref{Sec:UnifCF}, we would not expect the free slack approximation to hold, because $\fpar$ is not a constant function. It is unlikely that such an approximation exists for all non-uniform parent pdfs; a separate approximation would have to be devised for each parent. Poisson approximations for $\pcon$, $\pmon$, and $\psen$ are derived in \autoref{Subsec:nupmon}. The connectivity threshold and thermodynamic limit that are defined in \autoref{Sec:Estim} apply to $\fpar$  \cite{Matthew2003RGG}. We are unaware of an equivalent to the monitoring threshold of \autoref{eqn:shthresh} for $\fpar$. 
\section{Conclusions and Future Work}
\label{Sec:Conc}

We have presented an approach to characterizing and designing statistical properties of random multi-robot networks that are formed around boundaries by stochastic coverage schemes.  This work will enable the designer of a swarm robotic system to select the number of robots and the robot diameter, sensing range, and communication range that are guaranteed to yield target statistics of sensor coverage or communication connectivity around a boundary.  We have developed and validated this approach for robot configurations that are generated by uniform parent pdfs, which will occur in scenarios where there is a spatially homogeneous density of robots around the boundary.  We also extended the results to scenarios with robots that are distributed according to non-uniform parent pdfs.

%which we then formulated mathematically using the notion of a communication graph $\G$of a robot configuration. We subsequently computed explicit expressions for the properties  of $\G$ generated by uniform parents. We then applied  results from random graph theory to derive sharp thresholds for these properties. 

%These results were extended to general i.i.d. parents  %, and finally applied to design input parameters for uniform parents. 

%We first motivated the BC problem, which we then formulated mathematically using the notion of a communication graph $\G$of a robot configuration. We subsequently computed explicit expressions for the properties  of $\G$ generated by uniform parents. We then applied  results from random graph theory to derive sharp thresholds for these properties. These results were extended to general i.i.d. parents, and finally applied to design input parameters for uniform parents. 

We were able to develop these theoretical results because of the simplifying assumptions made in \autoref{Sec:Problem}: robots are homogeneous in terms of their parent pdf $\fpar$, diameter $\dia$, and communication/sensing range $\con$, and they can communicate perfectly within a disk of radius $\con$.  In practice, none of these assumptions may hold; hence, extending our results to scenarios with heterogenous robots and realistic communication is an important direction of our future work.  In addition, we will also consider problems where robots cover an area or volume as opposed to a one-dimensional boundary.  We describe our future work on these topics in more detail below.

\subsection{Heterogeneous $\con$, $\dia$, and parent pdfs}
\label{Subsec:hrange}

Suppose that each robot $\Rob_{i}$ has a distinct communication range $\con_i$.  Now the RGG $\G$ can be connected through \textit{multi-hop} messages, even when there are robots that are unable to communicate directly with \textit{at most one} of their neighbors. For example, consider three robots $\Rob_1$, $\Rob_2$, and $\Rob_3$ whose communication ranges satisfy $\con_{1} > \con_{2}$ and $\con_{3} > \con_{2}$. Suppose further that $\Rob_{1}$, but not $\Rob_{3}$, is within the range of $\Rob_{2}$. Then $\Rob_{2}$ may route packets destined for $\Rob_{3}$ through $\Rob_{1}$, thereby creating a connected $\G$. Such a scenario does not arise when each robot has the same range $\con$.  The inclusion of heterogeneous ranges $\con_i$ entails the loss of symmetry of the favorable region $\favn_{\mon}$ which was exploited in \autoref{Sec:Unif} to compute $\pcon$ in pseudo-polynomial time.  Our work in \cite{KumarDissert2016,KumarArxiv2016} shows that computing the probability that $\G$ is connected through 1-hop messages is $\#P$-Hard. Likewise, when the robots have distinct diameters $\dia_i$, the CF polytope $\ssimp_{CF}$ becomes a simplex with  hypercuboidal holes. Computing $\Vol(\ssimp_{CF})$ thus has the same complexity as finding the volume of a general simplex-hypercuboid intersection. 

The assumption of i.i.d parents is a major simplification that enables the analysis of the RGG $\G$. Since robot positions are not necessarily i.i.d. in general, their joint pdf $\jpos$ does not factorize into marginals as in \autoref{eqn:ostat}. When parent pdfs are distinct, the joint cdf of their order statistics is given by the \textit{Bapat-Beg} theorem, which requires the evaluation of a \textit{matrix permanent} \cite{ostat2003}. Since computing permanents can be $\#P$-Hard in general, determining  the properties of $\G$ becomes intractable for this case as well, forcing us to resort to approximations.

\subsection{Modeling realistic wireless communication}
\label{Subsec:realWC}
Wireless signals are electromagnetic waves, and consequently suffer from phenomena such as \textit{path loss, propagation loss}, and \textit{Rayleigh fading}. Furthermore, multiple Wi-fi transmitters that are placed close together will \textit{interfere} constructively or destructively with each other. These effects have been modeled using general PPPs \cite{haenggi2012stochastic}, but we do not know how these losses will affect an RGG $\G$. Including these effects in our models will involve a tradeoff between model expressiveness and tractability.   

\subsection{Boundary coverage (BC) in $\R^2$ and $\R^3$}
\label{Subsec:BChigh}
The geometric formulation of one-dimensional BC in \autoref{Sec:Geom} yielded closed-form formulas for the volumes of polytopes. We would not expect to be able to derive such formulas for two- and three-dimensional variants of BC, and instead would have to use the theory of RGGs much more extensively than we did here. Developing these results will allow us to design multi-robot systems for stochastic coverage tasks over terrestrial surfaces and within three-dimensional volumes in the air or underwater.

\IEEEpeerreviewmaketitle

\section*{Acknowledgments}
We gratefully acknowledge Dr. Theodore P. Pavlic and Sean Wilson for useful discussions.  
This work was supported in part by the DARPA YFA under Award No. D14AP00054.  Approved for public release; distribution is unlimited.  

% on the material presented in this paper

% Defense Advanced Research Projects Agency 

\ifCLASSOPTIONcaptionsoff
  \newpage
\fi



%\bibliographystyle{IEEEtran}

\bibliographystyle{IEEEtran} %plain}
\bibliography{refs}

\begin{thebibliography}{10}
\providecommand{\url}[1]{#1}
\csname url@samestyle\endcsname
\providecommand{\newblock}{\relax}
\providecommand{\bibinfo}[2]{#2}
\providecommand{\BIBentrySTDinterwordspacing}{\spaceskip=0pt\relax}
\providecommand{\BIBentryALTinterwordstretchfactor}{4}
\providecommand{\BIBentryALTinterwordspacing}{\spaceskip=\fontdimen2\font plus
\BIBentryALTinterwordstretchfactor\fontdimen3\font minus
  \fontdimen4\font\relax}
\providecommand{\BIBforeignlanguage}[2]{{%
\expandafter\ifx\csname l@#1\endcsname\relax
\typeout{** WARNING: IEEEtran.bst: No hyphenation pattern has been}%
\typeout{** loaded for the language `#1'. Using the pattern for}%
\typeout{** the default language instead.}%
\else
\language=\csname l@#1\endcsname
\fi
#2}}
\providecommand{\BIBdecl}{\relax}
\BIBdecl

\bibitem{Mavroidis2013}
C.~Mavroidis and A.~Ferreira, ``Nanorobotics: Past, present, and future,'' in
  \emph{Nanorobotics}, C.~Mavroidis and A.~Ferreira, Eds.\hskip 1em plus 0.5em
  minus 0.4em\relax Springer New York, 2013, pp. 3--27.

\bibitem{Diller2013}
E.~Diller and M.~Sitti, ``Micro-scale mobile robotics,'' \emph{Foundations and
  Trends in Robotics}, vol.~2, pp. 143--259, 2013.

\bibitem{bauer201425th}
S.~Bauer, S.~Bauer-Gogonea, I.~Graz, M.~Kaltenbrunner, C.~Keplinger, and
  R.~Schw{\"o}diauer, ``25th anniversary article: {A} soft future: {F}rom
  robots and sensor skin to energy harvesters,'' \emph{Advanced Materials},
  vol.~26, no.~1, pp. 149--162, 2014.

\bibitem{KubeBonabeau00}
C.~R. Kube and E.~Bonabeau, ``Cooperative transport by ants and robots,''
  \emph{Robot.\ Auton.\ Syst.}, vol.~30, no. 1--2, pp. 85--101, 2000.

\bibitem{OPGCMBD09}
R.~O'Grady, C.~Pinciroli, R.~Gro{\ss}, A.~L. Christensen, F.~Mondada,
  M.~Bonani, and M.~Dorigo, ``Swarm-bots to the rescue,'' in \emph{Proc. 10th
  European Conference on Artificial Life (ECAL)}, ser. LNCS, vol. 5777.\hskip
  1em plus 0.5em minus 0.4em\relax Budapest, Hungary: Springer-Verlag,
  Sept.~13--16, 2009, pp. 165--172.

\bibitem{Chen2015}
J.~Chen, M.~Gauci, W.~Li, A.~Kolling, and R.~Gro\ss, ``Occlusion-based
  cooperative transport with a swarm of miniature mobile robots,'' \emph{IEEE
  Trans. on Robotics}, pp. 1--15, March 2015.

\bibitem{Tong2013}
S.~Tong, E.~Fine, Y.~Lin, T.~J. Cradick, and G.~Bao, ``Nanomedicine: Tiny
  particles and machines give huge gains,'' \emph{Annals of Biomedical
  Engineering}, pp. 1--17, 2013.

\bibitem{Hauert2013}
S.~Hauert, S.~Berman, R.~Nagpal, and S.~N. Bhatia, ``A computational framework
  for identifying design guidelines to increase the penetration of targeted
  nanoparticles into tumors,'' \emph{Nano Today}, vol.~8, no.~6, pp. 566--576,
  2013.

\bibitem{Sinha2006}
R.~Sinha, G.~J. Kim, S.~Niel, and D.~M. Shin, ``Nanotechnology in cancer
  therapeutics: bioconjugated nanoparticles for drug delivery,'' \emph{Mol.
  Cancer Ther.}, vol.~5, Aug. 2006.

\bibitem{Wang2012}
S.~Wang and E.~E. Dormidontova, ``Selectivity of ligand-receptor interactions
  between nanoparticle and cell surfaces,'' \emph{Phys. Rev. Lett.}, vol. 109,
  p. 238102, Dec 2012.

\bibitem{long2013banana}
A.~W. Long, K.~C. Wolfe, M.~J. Mashner, and G.~S. Chirikjian, ``The banana
  distribution is {G}aussian: a localization study with exponential
  coordinates,'' \emph{Robotics: Science and Systems (RSS) VIII}, p. 265, 2013.

\bibitem{KumarICRA2014}
G.~P. Kumar and S.~Berman, ``Statistical analysis of stochastic multi-robot
  boundary coverage,'' in \emph{Proc. of the Int'l. Conf. on Robotics and
  Automation (ICRA)}, 2014, pp. 74--81.

\bibitem{Langmuir1918}
I.~Langmuir, ``The adsorption of gases on plane surfaces of glass, mica and
  platinum,'' \emph{J.\ Am.\ Chem.\ Soc.}, vol.~40, no.~9, pp. 1361--1403,
  1918.

\bibitem{Evans93}
J.~W. Evans, ``Random and cooperative sequential adsorption,'' \emph{Rev.\
  Mod.\ Phys.}, vol.~65, no.~4, pp. 1281--1329, October 1993.

\bibitem{TTVV00}
J.~Talbot, G.~Tarjus, P.~R. Van~Tassel, and P.~Viot, ``From car parking to
  protein adsorption: an overview of sequential adsorption processes,''
  \emph{Colloids Surfaces, A: Physicochem.\ Eng.\ Asp.}, vol. 165, no. 1--3,
  pp. 287--324, May 30, 2000.

\bibitem{Renyi58}
A.~R{\'{e}}nyi, ``On a one-dimensional problem concerning random
  space-filling,'' \emph{Publ.\ Math.\ Inst.\ Hung.\ Acad.\ Sci.}, vol.~3, pp.
  109--127, 1958.

\bibitem{SolomonWeiner86}
H.~Solomon and H.~Weiner, ``A review of the packing problem,'' \emph{Commun. in
  Stat.--Theory Methods}, vol.~15, no.~9, pp. 2571--2607, 1986.

\bibitem{Finch03}
S.~R. Finch, \emph{Mathematical Constants}, ser. Encyclopedia of Mathematics
  and its Applications.\hskip 1em plus 0.5em minus 0.4em\relax Cambridge
  University Press, 2003, vol.~94.

\bibitem{PavlicISRR2013}
T.~P. Pavlic, S.~Wilson, G.~P. Kumar, and S.~Berman, ``An enzyme-inspired
  approach to stochastic allocation of robotic swarms around boundaries,'' in
  \emph{Int'l. Symp. on Robotics Res. (ISRR)}, Singapore, 2013.

\bibitem{PavlicJSDSC14}
------, ``Control of stochastic boundary coverage by multirobot systems,''
  \emph{ASME J. Dyn. Sys. Meas. Control}, vol. 137, no.~3, p. 034505, Oct.
  2014.

\bibitem{Gurarie08}
E.~Gurarie, ``{Models and analysis of animal movements: From individual tracks
  to mass dispersal},'' Ph.D. dissertation, Univ. of Washington, 2008.

\bibitem{WilsonSwarmInt2014}
S.~Wilson, T.~P. Pavlic, G.~P. Kumar, A.~Buffin, S.~C. Pratt, and S.~Berman,
  ``\BIBforeignlanguage{English}{Design of ant-inspired stochastic control
  policies for collective transport by robotic swarms},''
  \emph{\BIBforeignlanguage{English}{Swarm Intelligence}}, vol.~8, no.~4, pp.
  303--327, 2014.

\bibitem{KumarSHS2013}
G.~P. Kumar, A.~Buffin, T.~P. Pavlic, S.~C. Pratt, and S.~M. Berman, ``A
  stochastic hybrid system model of collective transport in the desert ant
  \textit{{A}phaenogaster cockerelli},'' in \emph{Proc. 16th Int'l. Conf. on
  Hybrid Systems: Computation and Control (HSCC)}, 2013, pp. 119--124.

\bibitem{Wang2014}
C.~Wang, G.~Xie, and M.~Cao, ``Controlling anonymous mobile agents with
  unidirectional locomotion to form formations on a circle,''
  \emph{Automatica}, vol.~50, pp. 1100 -- 1108, 2014.

\bibitem{Brambilla2013}
M.~Brambilla, E.~Ferrante, M.~Birattari, and M.~Dorigo,
  ``\BIBforeignlanguage{English}{Swarm robotics: a review from the swarm
  engineering perspective},'' \emph{\BIBforeignlanguage{English}{Swarm
  Intelligence}}, vol.~7, no.~1, pp. 1--41, 2013.

\bibitem{CorrellHamann2015}
N.~Correll and H.~Hamann, ``Probabilistic modeling of swarming systems: From
  non-spatial to spatial dynamics,'' in \emph{Springer Handbook of
  Computational Intelligence}, J.~Kacprzyk and W.~Pedrycz, Eds., 2015.

\bibitem{Martinoli04}
A.~Martinoli, K.~Easton, and W.~Agassounon, ``Modeling swarm robotic systems: A
  case study in collaborative distributed manipulation,'' \emph{Int'l. Journal
  of Robotics Research}, vol.~23, pp. 415--436, 2004.

\bibitem{ref:Labella06}
T.~H. Labella, M.~Dorigo, and J.-L. Deneubourg, ``Division of labor in a group
  of robots inspired by ants' foraging behavior,'' \emph{ACM Trans. Auton.
  Adapt. Syst.}, vol.~1, no.~1, pp. 4--25, 2006.

\bibitem{Kanakia2014}
N.~C. A.~Kanakia, J.~Klingner, ``A response threshold sigmoid function model
  for swarm robot collaboration,'' in \emph{Proc. of the Int'l. Symp. on
  Distributed Autonomous Robotic Systems (DARS)}.\hskip 1em plus 0.5em minus
  0.4em\relax Springer Tracts on Advanced Robotics, 2014.

\bibitem{Correll07}
N.~Correll and A.~Martinoli, ``Modeling and optimization of a swarm-intelligent
  inspection system,'' in \emph{Proc. of the Int'l. Symp. on Distributed
  Autonomous Robotic Systems (DARS)}.\hskip 1em plus 0.5em minus 0.4em\relax
  Springer, 2007, pp. 369--378.

\bibitem{Davison2015Line}
P.~Davison, N.~Leonard, A.~Olshevsky, and M.~Schwemmer, ``Nonuniform line
  coverage from noisy scalar measurements,'' \emph{IEEE Trans. on Automatic
  Control}, vol.~60, no.~7, pp. 1975--1980, July 2015.

\bibitem{Frasca2015}
P.~Frasca, F.~Garin, B.~Gerencs\'er, and J.~M. Hendrickx, ``Optimal
  one-dimensional coverage by unreliable sensors,'' \emph{SIAM Journal on
  Control and Optimization}, vol.~53, no.~5, pp. 3120--3140, 2015.

\bibitem{Cortes2012}
J.~Cort\'es, ``Deployment of an unreliable robotic sensor network for spatial
  estimation,'' \emph{Systems \& Control Letters}, vol.~61, no.~1, pp. 41--49,
  2012.

\bibitem{friezeRandomGraphs}
A.~Frieze and M.~Karonski, \emph{Introduction to Random Graphs}.\hskip 1em plus
  0.5em minus 0.4em\relax Cambridge University Press, 2016.

\bibitem{Matthew2003RGG}
M.~Penrose, \emph{Random Geometric Graphs}, ser. Oxford Studies in
  Probability.\hskip 1em plus 0.5em minus 0.4em\relax Oxford University Press,
  2003.

\bibitem{barbour1992poisson}
A.~D. Barbour, L.~Holst, and S.~Janson, \emph{Poisson Approximation}.\hskip 1em
  plus 0.5em minus 0.4em\relax Clarendon Press Oxford, 1992.

\bibitem{ostat2003}
H.~David and H.~Nagaraja, \emph{Order Statistics}.\hskip 1em plus 0.5em minus
  0.4em\relax Hoboken, NJ, USA: John Wiley and Sons, 2003.

\bibitem{DyerFrieze}
M.~E. Dyer and A.~M. Frieze, ``On the complexity of computing the volume of a
  polyhedron,'' \emph{SIAM J. Comput.}, vol.~17, no.~5, pp. 967--974, 1988.

\bibitem{dyer1991computing}
------, ``Computing the volume of convex bodies: a case where randomness
  provably helps,'' \emph{Probabilistic Combinatorics and its Applications},
  vol.~44, pp. 123--170, 1991.

\bibitem{franceschetti2003percolation}
M.~Franceschetti, L.~Booth, M.~Cook, R.~Meester, and J.~Bruck, ``Percolation in
  multi-hop wireless networks,'' IEEE Trans. on Information Theory, Tech. Rep.,
  2003.

\bibitem{xue2004number}
F.~Xue and P.~R. Kumar, ``The number of neighbors needed for connectivity of
  wireless networks,'' \emph{Wireless Networks}, vol.~10, no.~2, pp. 169--181,
  2004.

\bibitem{hekmat2003degree}
R.~Hekmat and P.~Van~Mieghem, ``Degree distribution and hopcount in wireless
  ad-hoc networks,'' in \emph{Proc. of the IEEE Int'l. Conf. on Networks
  (ICON)}, Sydney, Australia, 2003, pp. 603--609.

\bibitem{haenggi2012stochastic}
M.~Haenggi, \emph{Stochastic Geometry for Wireless Networks}.\hskip 1em plus
  0.5em minus 0.4em\relax Cambridge University Press, 2012.

\bibitem{godehardt1996connectivity}
E.~Godehardt and J.~Jaworski, ``On the connectivity of a random interval
  graph,'' \emph{Random Structures \& Algorithms}, vol.~9, no. 1-2, pp.
  137--161, 1996.

\bibitem{goel2004sharp}
A.~Goel, S.~Rai, and B.~Krishnamachari, ``Sharp thresholds for monotone
  properties in random geometric graphs,'' in \emph{Proc. of the 36th Annual
  ACM Symposium on Theory of Computing (STOC)}, 2004, pp. 580--586.

\bibitem{ConvPoly}
B.~Gr{\"{u}}nbaum, \emph{Convex Polytopes}, 2nd~ed.\hskip 1em plus 0.5em minus
  0.4em\relax Springer-Verlag New York Inc., 2002.

\bibitem{sage}
\BIBentryALTinterwordspacing
{T}he Sage Development Team (2016), {S}age{M}ath: {O}pen {S}ource
  {M}athematical {S}ystem, ``{P}olyhedra''. [Online]. Available:
  \url{http://doc.sagemath.org/html/en/reference/geometry/sage/geometry/polyhedron/constructor.html}
\BIBentrySTDinterwordspacing

\bibitem{cddlib}
\BIBentryALTinterwordspacing
K.~Fukuda. (2015) cddlib {P}olyhedral {C}omputation {L}ibrary. [Online].
  Available: \url{https://www.inf.ethz.ch/personal/fukudak/cdd_home/index.html}
\BIBentrySTDinterwordspacing

\bibitem{vanLintCombin2001}
J.~van Lint and R.~Wilson, \emph{A Course in Combinatorics}.\hskip 1em plus
  0.5em minus 0.4em\relax New York, NY, USA: Cambridge University Press, 2001.

\bibitem{muthukrishnan2005bin}
S.~Muthukrishnan and G.~Pandurangan, ``The bin-covering technique for
  thresholding random geometric graph properties,'' in \emph{Proc. of the 16th
  Annual ACM-SIAM Symposium on Discrete Algorithms}, 2005, pp. 989--998.

\bibitem{yi2010time}
S.~Yi, \emph{Time-Delay Systems: Analysis and Control using the Lambert W
  Function}.\hskip 1em plus 0.5em minus 0.4em\relax World Scientific, 2010.

\bibitem{kumar_polytope_2016}
\BIBentryALTinterwordspacing
G.~P. Kumar, ``Polytope {Computations} {Repository},'' Jun 2016. [Online].
  Available: \url{https://gitlab.com/thedragondraco/polytope-comps}
\BIBentrySTDinterwordspacing

\bibitem{kumarUniformCT2016}
\BIBentryALTinterwordspacing
------, ``Solver for uniform {CT} attachments,'' Jun 2016. [Online]. Available:
  \url{https://gitlab.com/thedragondraco/polytope-comps/blob/master/numsolve/unif-ct.py}
\BIBentrySTDinterwordspacing

\bibitem{kumarUniformCF2016}
\BIBentryALTinterwordspacing
------, ``Solver for uniform {CF} attachments,'' Jun 2016. [Online]. Available:
  \url{https://gitlab.com/thedragondraco/polytope-comps/blob/master/numsolve/unif-cf.py}
\BIBentrySTDinterwordspacing

\bibitem{KumarArxiv2016}
\BIBentryALTinterwordspacing
G.~P. Kumar and S.~Berman, ``Probabilistic analysis of the communication
  network created by dynamic boundary coverage,'' \emph{arXiv}, 2016. [Online].
  Available: \url{https://arxiv.org/abs/1604.01452}
\BIBentrySTDinterwordspacing

\bibitem{KumarDissert2016}
G.~P. Kumar, ``Development and analysis of stochastic boundary coverage
  strategies for multi-robot systems,'' Ph.D. dissertation, Arizona State
  University, 2016.

\end{thebibliography}

\vspace{0cm}

\begin{IEEEbiography}[{\includegraphics[width=1in,height=1.25in,clip,keepaspectratio]{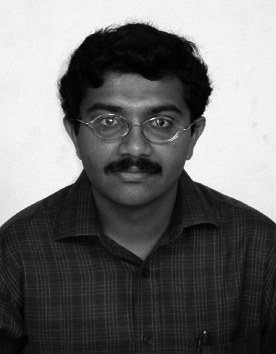}}]{Ganesh P. Kumar}
(M'12) received the B.Tech. degree in computer science from Kerala University, Trivandrum, Kerala, India, in 2004 and the M.Tech. degree in computer science from the International Institute of Information Technology, Hyderabad, India, in 2006.  He received the Ph.D. degree in computer science and engineering from Arizona State University, Tempe, AZ in 2016.

From 2006 to 2010, he held positions as a Software Engineer at CA Technologies and a Senior Software Engineer at Yahoo! and the erstwhile Motorola Mobility Solutions.  He is currently an Assistant Research Scientist in Prof. Spring Berman's group at Arizona State University.  His research focuses on computational problems in swarm robotics, and he is interested in both the theoretical and practical aspects of designing software for robotic swarms.  
  
Dr. Kumar is a member of ACCU, the Association for Computing Machinery (ACM), and the American Helicopter Society (AHS).  He was a recipient of the Government of India's Department of Science and Technology (DST) award for the period 2005-2006.
\end{IEEEbiography}

\vspace{-0.5cm}

\begin{IEEEbiography}[{\includegraphics[width=1in,height=1.25in,clip,keepaspectratio]{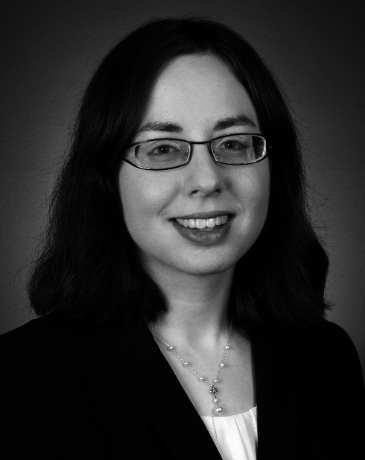}}]{Spring Berman} 
(M'07) received the B.S.E. degree in mechanical and aerospace engineering from Princeton University, Princeton, NJ, in 2005 and the M.S.E. and Ph.D. degrees in mechanical engineering and applied mechanics from the University of Pennsylvania, Philadelphia, PA, in 2008 and 2010, respectively.

From 2010 to 2012, she was a Postdoctoral Researcher in Computer Science at Harvard University, Cambridge, MA.  Since 2012, she has been an Assistant Professor of Mechanical and Aerospace Engineering with the School for Engineering of Matter, Transport and Energy (SEMTE), Arizona State University, Tempe, AZ.  Her research focuses on the modeling and analysis of behaviors in biological and engineered collectives and the synthesis of control strategies for robotic swarms.

Prof. Berman is a recipient of the 2014 Defense Advanced Research Projects Agency Young Faculty Award and the 2016 Office of Naval Research Young Investigator Award.
\end{IEEEbiography}

\end{document}